\documentclass[11pt,a4paper]{article}
\usepackage{layaureo}

\usepackage[utf8]{inputenc}
\usepackage{times}
\usepackage{amsfonts}
\usepackage{amsmath}
\usepackage{amssymb}
\usepackage{mathtools}
\usepackage{amsthm}
\usepackage{enumerate}
\usepackage{setspace}
\usepackage{xparse}
\usepackage{hyperref}
\usepackage{url}
\usepackage{relsize}
\usepackage{stackengine}
\stackMath

\usepackage{tikz}
\usepackage{tikz-qtree}
\usetikzlibrary{arrows,automata,positioning}

\newcommand{\Absent}{\ensuremath{\mbox{\textit{Absent}}}}
\newcommand{\absent}{\ensuremath{\mbox{\textit{absent}}}}
\newcommand{\Bacteria}{\ensuremath{\mbox{\textit{Bacteria}}}}
\newcommand{\bacteria}{\ensuremath{\mbox{\textit{bacteria}}}}
\newcommand{\Coin}{\ensuremath{\mbox{\textit{Coin}}}}
\newcommand{\coin}{\ensuremath{\mbox{\textit{coin}}}}
\newcommand{\GoOut}{\ensuremath{\mbox{\textit{GoOut}}}}
\newcommand{\HasKeys}{\ensuremath{\mbox{\textit{HasKeys}}}}
\newcommand{\Heads}{\ensuremath{\mbox{\textit{Heads}}}}
\newcommand{\heads}{\ensuremath{\mbox{\textit{heads}}}}

\newcommand{\Inside}{\ensuremath{\mbox{\textit{Inside}}}}

\newcommand{\Location}{\ensuremath{\mbox{\textit{Location}}}}
\newcommand{\LockedOut}{\ensuremath{\mbox{\textit{LockedOut}}}}
\newcommand{\Outside}{\ensuremath{\mbox{\textit{Outside}}}}
\newcommand{\PickupKeys}{\ensuremath{\mbox{\textit{PickupKeys}}}}

\newcommand{\Present}{\ensuremath{\mbox{\textit{Present}}}}
\newcommand{\present}{\ensuremath{\mbox{\textit{present}}}}
\newcommand{\Rash}{\ensuremath{\mbox{\textit{Rash}}}}
\newcommand{\rash}{\ensuremath{\mbox{\textit{rash}}}}
\newcommand{\Resistant}{\ensuremath{\mbox{\textit{Resistant}}}}
\newcommand{\resistant}{\ensuremath{\mbox{\textit{resistant}}}}

\newcommand{\See}{\ensuremath{\mbox{\textit{See}}}}

\newcommand{\Tails}{\ensuremath{\mbox{\textit{Tails}}}}
\newcommand{\tails}{\ensuremath{\mbox{\textit{tails}}}}
\newcommand{\TakesMedicine}{\ensuremath{\mbox{\textit{TakesMedicine}}}}
\newcommand{\takesMedicine}{\ensuremath{\mbox{\textit{takesMedicine}}}}

\newcommand{\Toss}{\ensuremath{\mbox{\textit{Toss}}}}
\newcommand{\toss}{\ensuremath{\mbox{\textit{toss}}}}
\newcommand{\Weak}{\ensuremath{\mbox{\textit{Weak}}}}
\newcommand{\weak}{\ensuremath{\mbox{\textit{weak}}}}

\newcommand{\action}{\ensuremath{\mbox{\textit{action}}}}
\NewDocumentCommand\Bel{ g g g }{ \ensuremath{\textit{Bel}\IfNoValueTF{#1}{}{\IfNoValueTF{#2}{_{#1}}{ \IfNoValueTF{#3}{_{\langle #1, #2 \rangle}}{_{\langle #1, #2, #3 \rangle}}  } } } }

\newcommand{\belongsTo}{\ensuremath{\mbox{\textit{belongsTo}}}}
\newcommand{\body}{\ensuremath{\mbox{\textit{body}}}}
\newcommand{\botProgram}[2]{\textit{bot}_{#1} (#2)}

\newcommand{\causesOneOf}{\ensuremath{\mbox{ {\bf causes-one-of} }}}
\newcommand{\causesOutcome}{\ensuremath{\mbox{\textit{causesOutcome}}}}
\newcommand{\choiceElement}[3]{(#1,#2)@#3}
\newcommand{\cprop}{\ensuremath{\textit{cprop}}}
\newcommand{\definitelyPerformed}{\ensuremath{\mbox{\textit{definitelyPerformed}}}}

\newcommand{\dom}{\textit{dom}}
\newcommand{\ec}{\ensuremath{\mbox{\textit{ec}}}}
\newcommand{\effectChoice}{\ensuremath{\mbox{\textit{effectChoice}}}}

\newcommand{\Ev}{\ensuremath{\mbox{\textit{Ev}}}}
\newcommand{\eval}{\ensuremath{\mbox{\textit{eval}}}}
\newcommand{\evaluation}[2]{\ensuremath{\mbox{\textit{e}}_{#1} (#2)}}
\newcommand{\false}{\ensuremath{\mbox{\textit{false}}}}
\newcommand{\False}{\bot}
\newcommand{\fluent}{\ensuremath{\mbox{\textit{fluent}}}}
\newcommand{\fluentState}[1]{\tilde{#1}}
\newcommand{\fluentOrAction}{\ensuremath{\mbox{\textit{fluentOrAction}}}}

\newcommand{\hasInstant}{\ensuremath{\mbox{\textit{hasInstant}}}}
\newcommand{\hasInstants}{\ensuremath{\mbox{\textit{hasInstants}}}}
\newcommand{\head}{\ensuremath{\mbox{\textit{head}}}}
\newcommand{\holds}{\ensuremath{\mbox{\textit{holds}}}}
\newcommand{\hourglass}{\rotatebox{90}{$\bowtie$}}
\newcommand{\holdsWithProb}{\ensuremath{\mbox{ {\bf holds-with-prob }}}}

\newcommand{\ic}{\ensuremath{\mbox{\textit{ic}}}}
\newcommand{\id}{\ensuremath{\mbox{\textit{id}}}}
\newcommand{\ifBelieves}{\ensuremath{\mbox{ {\bf if-believes} }}}

\newcommand{\iLiteral}{\ensuremath{\mbox{\textit{iLiteral}}}}

\newcommand{\initialChoice}{\ensuremath{\mbox{\textit{initialChoice}}}}
\newcommand{\initialChoiceSymbol}{\hourglass}

\newcommand{\initiallyOneOf}{\ensuremath{\mbox{{\bf initially-one-of} }}}
\newcommand{\initialCondition}{\ensuremath{\mbox{\textit{initialCondition}}}}
\newcommand{\initialInstant}{\ensuremath{\mbox{$\bar{0}$}}}
\newcommand{\inOcc}{\ensuremath{\mbox{\textit{inOcc}}}}
\newcommand{\instant}{\ensuremath{\mbox{\textit{instant}}}}

\newcommand{\literal}{\ensuremath{\mbox{\textit{literal}}}}
\NewDocumentCommand\matr{m g g}{ \ensuremath{\mathbf{#1}\IfNoValueTF{#2}{}{ \IfNoValueTF{#3}{_{#2}}{_{#2,#3}} } } }
\newcommand{\maxinst}{\ensuremath{\mbox{\textit{maxinst}}}}
\NewDocumentCommand\model{ g g }{ \ensuremath{M\IfNoValueTF{#1}{}{\IfNoValueTF{#2}{_{#1}}{ _{#1} (#2) }}}}
\NewDocumentCommand\modelStar{ g g }{ \ensuremath{M^{*}\IfNoValueTF{#1}{}{\IfNoValueTF{#2}{_{#1}}{ _{#1} (#2) }}}}
\newcommand{\narr}{\textit{narr}}
\DeclareRobustCommand{\mmodels}{\mathrel{||}\joinrel\Relbar}
\DeclareRobustCommand{\notmmodels}{\mathrel{||}\joinrel\not\Relbar}
\newcommand{\occ}{\ensuremath{\mbox{\textit{occ}}}}

\newcommand{\performed}{\ensuremath{\mbox{\textit{performed}}}}

\newcommand{\performedAt}{\ensuremath{\mbox{ {\bf performed-at} }}}

\newcommand{\possiblyPerformed}{\ensuremath{\mbox{\textit{possiblyPerformed}}}}
\newcommand{\possVal}{\ensuremath{\mbox{\textit{possVal}}}}

\renewcommand{\restriction}{\mathord{\upharpoonright}}

\newcommand{\senses}{\ensuremath{\mbox{ {\bf senses} }}}

\newcommand{\topProgram}[2]{\textit{top}_{#1} (#2)}
\newcommand{\tr}{tr}
\newcommand{\true}{\ensuremath{\mbox{\textit{true}}}}
\newcommand{\True}{\top}

\newcommand{\takesValues}{\ensuremath{\mbox{ {\bf takes-values} }}}
\newcommand{\tset}{\textit{tset}}
\newcommand{\update}[2]{\ensuremath{\mbox{$#1 \oplus #2$}}}
\newcommand{\vals}{\ensuremath{\mbox{\textit{vals}}}}

\newcommand{\weight}{\pi}
\newcommand{\withAccuracies}{\ensuremath{\mbox{ {\bf with-accuracies} }}}
\newcommand{\withProb}{\ensuremath{\mbox{ {\bf with-prob} }}}

\newcommand{\lit}[2]{\mbox{\ensuremath{ #1 \! = \! #2}}}
\newcommand{\ifor}[2]{\ensuremath{[#1]@#2}}

\newlength{\gapBeforeAxioms}
\newlength{\gapBetweenAxioms}
\newlength{\gapAfterAxioms}
\newlength{\axiomsIndent}
\newlength{\axiomsMedIndent}
\newlength{\axiomsBigIndent}
\theoremstyle{definition}
\newtheorem{proposition}{Proposition}

\theoremstyle{definition}
\newtheorem{corollary}{Corollary}

\theoremstyle{definition}
\newtheorem{scenario}{Scenario}

\theoremstyle{definition}
\newtheorem{example}{Example}

\theoremstyle{definition}

\theoremstyle{definition}
\newtheorem{definition}{Definition}

\theoremstyle{definition}
\newtheorem{notation}{Notation}

\theoremstyle{definition}
\newtheorem{lemma}{Lemma}

\theoremstyle{definition}

\begin{document}

\date{}
\title{\textbf{Foundations for a Probabilistic Event Calculus} \\\emph{Technical Report}}
\author{Fabio A. D'Asaro, Antonis Bikakis, Luke Dickens and Rob Miller\\
				{\it \normalsize \{\href{mailto:uczcfad@ucl.ac.uk}{uczcfad}, \href{mailto:a.bikakis@ucl.ac.uk}{a.bikakis}, \href{mailto:l.dickens@ucl.ac.uk}{l.dickens}, \href{mailto:r.s.miller@ucl.ac.uk}{r.s.miller}\}@ucl.ac.uk}\\ \\
				{\normalsize Department of Information Studies}\\
				{\normalsize University College London}}
\maketitle

\sloppy

\section{Introduction}

%
The Event Calculus (EC) \cite{kowalski1989logic} is a well-known approach to reasoning about the effects of a narrative of action occurrences (events) along a time line. This paper describes PEC, an adaptation of EC able to reason with probabilistic causal knowledge.
There are numerous applications for this kind of probabilistic reasoning, e.g.\ in modelling medical, environmental, legal and commonsense domains, and in complex activity recognition and security monitoring.
PEC's main characteristics are: i) it supports EC-style narrative reasoning, ii) it uses a \emph{possible worlds} semantics to naturally allow for \emph{epistemic} extensions, iii) it uses a tailored action language syntax and semantics,
iv) its generality allows in principle for the use of other models of uncertainty, e.g.\ Fuzzy Logic \cite{zadeh1965fuzzy} or Dempster-Schafer Theory \cite{dempster1967upper}, and
v) for a wide subset of domains it has a sound and complete ASP implementation.

Although other formalisms exist for probabilistic reasoning about actions, PEC is, to our knowledge, the only framework to combine these features.
We use the following two example scenarios to illustrate the main definitions and characteristics of our framework.

\begin{scenario}[Coin Toss] \label{scenario:coin}
	A coin initially (instant $0$) shows Heads. A robot can attempt to toss the coin, but there is a small chance that it will fail to pick it up, leaving the coin unchanged. The robot attempts to toss the coin (instant $1$).
\end{scenario}

\begin{scenario}[Antibiotic] \label{scenario:antibiotic}
	A patient has a rash often associated with a bacterial infection, and can take an antibiotic known to be reasonably effective. Treatment is not always successful, and if not may still clear the rash. Failed treatment leaves the bacteria resistant. The patient is treated twice (instants $1$ and $3$).
\end{scenario}

\begin{scenario}[Keys] \label{scenario:keys}
	Leaving the house without first picking up the keys causes being locked out. In the context of a daily routine, there is a small chance that a person forgets to pick up the keys before leaving the house at 7:40 AM.
\end{scenario}
\section{PEC}

\subsection{Syntax}
\begin{definition}[Domain Language] A \emph{domain language} is a tuple $\mathcal{L} = \langle \mathcal{F}, \mathcal{A}, \mathcal{V}, \vals, \mathcal{I}, \leq, \initialInstant \rangle $ consisting of a finite non-empty set $\mathcal{F}$ of \emph{fluents}, a finite set $\mathcal{A}$ of \emph{actions}, a finite non-empty set $\mathcal{V}$ of \emph{values} such that $\{ \True, \False \} \subseteq \mathcal{V}$, a function $\vals:\mathcal{F}\cup\mathcal{A} \rightarrow 2^{\mathcal{V}} \setminus \emptyset$, a non-empty set $\mathcal{I}$ of \emph{instants} and a minimum element $\initialInstant \in \mathcal{I}$ w.r.t. a total ordering $\leq$ over $\mathcal{I}$. For $A\in \mathcal{A}$ we impose $\vals(A)=\{ \True, \False \}$.
\end{definition}

\begin{example}
	An appropriate domain language for Scenario \ref{scenario:coin} would be $ \mathcal{L}_C = \langle \mathcal{F}_C, \mathcal{A}_C, \mathcal{V}_C, \vals_C, \mathbb{N}, \leq_\mathbb{N}, 0 \rangle$ where, $\mathcal{F}_C = \{ \Coin \}$, $ \mathcal{A}_C = \{ \Toss \} $, $\mathcal{V}_C = \{ \True, \False, \Heads, \Tails \}$, $\vals_C(\Coin)=\{ \Heads, \Tails \}$ and $\vals_C(\Toss)=\{ \True, \False \}$, $\mathbb{N}$ is the set of natural numbers (including 0), and $\leq_\mathbb{N}$ is the standard total ordering between naturals. Scenario \ref{scenario:antibiotic} could be captured by a language $ \langle \mathcal{F}_A, \mathcal{A}_A, \mathcal{V}_A, \vals_A, \mathbb{N}, \leq_\mathbb{N}, 0 \rangle$ where $\mathcal{F}_A= \{ \Bacteria, \Rash \}$, $ \mathcal{A}_A = \{ \TakesMedicine \} $, $\mathcal{V}_A = \{ \True, \False, \Weak, \Present, \Resistant, \Absent \}$,  and $ \vals_A $ is defined by $ \vals_A ( \Bacteria ) = \{ \Weak, \Resistant, \Absent \} $ and $ \vals_A ( \Rash ) = \{ \Present, \Absent \} $.
\end{example}

In what follows, all definitions are with respect to a domain language $ \mathcal{L} = \langle \mathcal{F}, \mathcal{A}, \mathcal{V}, \vals, \mathcal{I}, \leq, \initialInstant \rangle$.

We begin by defining what a (fluent) literal and a formula are in this language. Literals and formulas with time information attached are called i-literals and i-formulas respectively.

\begin{definition}[Fluent and Action Literals, i-literals] A \emph{fluent literal} is an expression of the form $\lit{F}{V}$ for some $F\in \mathcal{F}$ and $V\in \vals(F)$. A fluent is \emph{boolean} if $\vals(F)=\{ \True, \False \}$. An \emph{action literal} is either $\lit{A}{\True}$ or $\lit{A}{\False}$. When no ambiguity can arise, $\lit{Z}{\True}$ is sometimes abbreviated to $Z$ and $\lit{Z}{\False}$ is abbreviated to $\neg Z$ for $Z$ a fluent or action. An \emph{i-literal} is an expression of the form $\ifor{L}{I}$ for some (fluent or action) literal $L$ and some $I\in \mathcal{I}$.
\end{definition}

\begin{definition}[Formulas, i-formulas]  The set of \emph{formulas}, denoted by $\Theta$, is the closure of the set of literals under $\wedge$ and $\neg$ (with $\vee$ and $\rightarrow$ being defined as shorthand in the usual way). The set of \emph{i-formulas}, denoted by $\Phi$, is the closure of the set of i-literals under $\wedge$ and $\neg$. We use the shorthand $\ifor{\theta}{I}$ for the i-formula formed from the formula $\theta$ and the instant $I$ by replacing all literals $L$ occurring in $\theta$ by $\ifor{L}{I}$, e.g. $\ifor{\lit{F}{V} \rightarrow \lit{F'}{V'}}{3}$ is a shorthand for $\ifor{\lit{F}{V}}{3} \rightarrow \ifor{\lit{F'}{V'}}{3}$.
\end{definition}

\begin{example}
In Scenario \ref{scenario:coin} the i-literal $ \ifor{\lit{\Coin}{\Heads}}{3} $ indicates that the coin shows heads at instant 3, while $ \ifor{\neg \Toss}{2} $ indicates that the robot does not attempt to toss the coin at instant 2. In Scenario \ref{scenario:antibiotic}, $\ifor{\lit{\Rash} {\Present}}{0} \wedge \ifor{\lit{\Bacteria}{\Absent} \wedge \TakesMedicine}{3} $ indicates that the patient initially has a rash, and that she takes the medicine and the bacterial infection is absent at instant $3$.
\end{example}

\begin{definition}[State, Partial State, Fluent State] A \emph{state} $S$ is a set of literals, exactly one for each $F \in \mathcal{F}$ and $A\in \mathcal{A}$. A \emph{partial state} is a subset $X \subseteq S$ of a state $S$. Given a partial state $X$, we call its subset containing all and only the fluent literals in $X$ a \emph{partial fluent state}, and we denote it as $X \restriction \mathcal{F}$. For a state $S$ we call $S\restriction \mathcal{F}$ a \emph{fluent state}. We also define $X\restriction \mathcal{A}$ as the subset of $X$ containing all and only the action literals in $X$. The set of all states is denoted by $\mathcal{S}$, the set of all partial states is denoted by ${\cal X}$, and we use $\tilde{\mathcal{S}}$ and $\tilde{\cal X}$ to denote the sets $\{ S\restriction \mathcal{F} \mid S \in \mathcal{S} \}$ and $\{ X\restriction \mathcal{F} \mid X\in \mathcal{X} \}$ respectively. 
\end{definition}

\begin{example}
	One of the states we can build with the elements of the domain language $ \langle \mathcal{F}_A, \mathcal{A}_A, \mathcal{V}_A, \vals_A, \mathbb{N}, \leq_\mathbb{N}, 0 \rangle$ for Scenario \ref{scenario:antibiotic} is $ S^1_{A} = \{ \lit{\Bacteria}{\Resistant}, \lit{\Rash}{\Absent}, \neg \TakesMedicine \} $. Its associated fluent state is $ S^1_A \restriction \mathcal{F} = \{ \lit{\Bacteria}{\Resistant}, \lit{\Rash}{\Absent} \} $. Any arbitrary subset of $S^1_A$, e.g. $X^1_A = \{ \lit{\Rash}{\Absent}, \neg \TakesMedicine \}$, is a partial state, whereas any arbitrary subset of $ S_A^1 \restriction \mathcal{F} $, e.g. $ X_A^1 \restriction\mathcal{F} = \{ \lit{\Rash}{\Absent} \} $, is a partial fluent state.
\end{example}

\begin{definition}[Outcome, Projection Functions]\label{def:outcomes}
	An \emph{outcome} is a pair of the form $(\fluentState{X},P^+)$ for some $ \fluentState{X} \in \fluentState{\mathcal{X}} $ and $ P^+ \in (0,1] $. The two \emph{projection functions} $\chi$ and $\pi$ are such that $\chi( (\tilde{X},P^+) ) = \tilde{X}$ and $\pi( (\tilde{X},P^+) ) = P^+$ for any outcome. The set of all outcomes $\tilde{\mathcal{X}} \times (0,1]$ will be denoted by $ \mathcal{O} $.
\end{definition}

\begin{definition}[Weight of a Set of Outcomes]\label{def:weightOfASetOfOutcomes} Given a finite set of outcomes \[ B=\left\{ O_1, O_2, \dots, O_m \right \} \] we define the \emph{weight of O} as \[ \weight(B)=\sum_{i=1}^{m} \pi(O_i). \]
\end{definition}

\begin{notation} In the following, we will generally use:\\
	\noindent\hspace*{\axiomsIndent}$I, I', I_1, I'', I_2, \dots$ to denote elements of $\mathcal{I}$,\\
	\noindent\hspace*{\axiomsIndent}$A, A', A_1, A'', A_2, \dots$ to denote elements of $\mathcal{A}$,\\
	\noindent\hspace*{\axiomsIndent}$F, F', F_1, F'', F_2, \dots$ to denote elements of $\mathcal{F}$,\\
	\noindent\hspace*{\axiomsIndent}$V, V', V_1, V'', V_2, \dots$ to denote elements of $\mathcal{V}$,\\
	\noindent\hspace*{\axiomsIndent}$\theta, \theta', \theta_1, \theta'', \theta_2, \dots$ to denote formulas,\\
	\noindent\hspace*{\axiomsIndent}$\varphi, \varphi', \varphi_1, \varphi'', \varphi_2, \dots$ to denote i-formulas,\\
	\noindent\hspace*{\axiomsIndent}$P, P', P_1, P'', P_2, \dots$ to denote real values in $[0,1]$,\\
	\noindent\hspace*{\axiomsIndent}$P^{+}, P^{+}_1, P^{+}_2, \dots$ to denote real values in $(0,1]$,\\
	\noindent\hspace*{\axiomsIndent}$S, S', S_1, S'', S_2, \dots$ to denote elements of $\mathcal{S}$,\\
	\noindent\hspace*{\axiomsIndent}$X, X', X_1, X'', X_2, \dots$ to denote elements of $\mathcal{X}$,\\
	\noindent\hspace*{\axiomsIndent}$\fluentState{S}, \fluentState{S}', \fluentState{S}_1, \fluentState{S}'', \fluentState{S}_2, \dots$ to denote elements of $\fluentState{\cal S}$,\\
	\noindent\hspace*{\axiomsIndent}$\fluentState{X}, \fluentState{X}', \fluentState{X}_1, \fluentState{X}'', \fluentState{X}_2, \dots$ to denote elements of $\fluentState{\cal X}$,\\
	\noindent\hspace*{\axiomsIndent}$O, O', O_1, O'', O_2, \dots$ to denote outcomes.
\end{notation}

We now introduce the standard propositions of our language: v-propositions are used to declare which value a fluent may take, c-propositions are used to model the causal relationships of a domain, i-propositions declare the initial conditions, p-propositions are for the action occurrences, and h-propositions state that a given i-formula holds.

\begin{definition}[v-proposition]
A \emph{v-proposition} has the form
\begin{equation}\label{eq:vProposition}
F \takesValues \{ V_1, \dots, V_m \}
\end{equation} where $m\geq 1$ and $\{V_1, \dots, V_m\}=vals(F)$.
\end{definition}

\begin{definition}[c-proposition, Head and Body of a c-proposition]
A \emph{c-proposition} $c$ has the form
\begin{equation}\label{eq:cProposition}
\theta \causesOneOf \{ O_1, O_2, \dots, O_m \}
\end{equation}
where $O_i \in \mathcal{O}$, $\chi(O_i) \neq \chi(O_j)$ when $i\neq j$, $\theta$ is a formula such that $\theta$ Herbrand-entails\footnote{For two formulas $ \theta $ and $ \theta' $ we write that $ \theta $ Herbrand-entails $ \theta' $ if, taking literals as propositions, every classical Herbrand model of $\theta$ is also a Herbrand model of $\theta'$.} $\lit{A}{\True}$ for at least one $A\in \mathcal{A}$, and $\weight(\left\{ O_1, \dots, O_m \right\})=1$. 
$\body(C) \!=\! \theta$ and $\head(C) \!=\! \left\{ O_1, \dots, O_m \right\}$ are the \emph{body} and \emph{head} of $C$, respectively.  We often omit $O_i$ from $\head(C)$ if $\chi(O_i)=\emptyset$ (leaving it implicit since $\weight(\left\{ O_1, \dots, O_m \right\})=1$).
\end{definition}

\begin{definition}[i-proposition]\label{def:iProposition}
An \emph{i-proposition} has the form
\begin{equation}\label{eq:iProposition}
\initiallyOneOf \{ O_1, O_2, \dots, O_m \}
\end{equation}
where $O_i \!\in \!\mathcal{O}$, $\weight(\{ O_1, \dots, O_m \})\!=\!1$, $ \chi(O_i) \!\in\! \fluentState{\mathcal{S}}$, and $\chi(O_i) \!\neq\! \chi(O_j)$ when $i\!\neq\! j$.
\end{definition}

\begin{definition}[p-proposition]
A \emph{p-proposition} has the form
\begin{equation}\label{eq:pProposition}
A \performedAt I \withProb P^+
\end{equation}
where $P^+ \in (0,1]$ and $I$ is such that $I<I'$ for some other $I'\in\mathcal{I}$. When a p-proposition $p$ has the form \eqref{eq:pProposition} we say that \emph{$p$ has instant $I$}.

In the following, we will frequently use
\begin{equation*}
A \performedAt I
\end{equation*}
as a shorthand for the p-proposition
\begin{equation*}
A \performedAt I \withProb 1.
\end{equation*}

\end{definition}

\begin{notation} In the following, we will generally use lowercase letters to denote propositions, e.g. $ c, c', c_1, c'', c_2, \dots $ will be used for c-propositions.
\end{notation}

\begin{definition}[Domain Description]\label{def:domaindescription} A \emph{domain description} is a finite set ${\cal D}$ of v-propositions, c-propositions, p-propositions and i-propositions such that: (i) for any two distinct c-propositions in $\mathcal{D}$ with bodies $\theta$ and $\theta'$ respectively, $ \theta $ does not Herbrand-entail $ \theta' $, (ii) $\mathcal{D}$ contains exactly one i-proposition, (iii) $ \mathcal{D} $ contains exactly one v-proposition for each $F\in\mathcal{F}$ and (iv) if a p-proposition ``$ A \performedAt I \withProb P' $'' belongs to $ \mathcal{D} $, then there is no other p-proposition of the form ``$ A \performedAt I \withProb P'' $'' for some $ P''\in(0,1] $ that belongs to $ \mathcal{D} $.
\end{definition}

\begin{definition}[Action Narrative]\label{def:narrative} An \emph{action narrative} is any finite set of p-propositions. For $ \mathcal{D} $ a domain description, we define the action narrative $\narr(\mathcal{D})$ as the set of all p-propositions in $ \mathcal{D} $.
\end{definition}

\begin{example}\label{ex:coindomain} Scenario \ref{scenario:coin} can be modeled using the following domain description $ \mathcal{D}_C $:\\

\noindent
\hspace*{\axiomsIndent}$ \Coin \takesValues \{ \Heads, \Tails \} $\hfill(C1)\\

\noindent
\hspace*{\axiomsIndent}$ \initiallyOneOf \{ (\{\Coin=\Heads\}, 1) \}$ \hfill(C2)\\

\noindent
\hspace*{\axiomsIndent}$ \Toss \causesOneOf$ \hfill (C3)\\
\hspace*{\axiomsMedIndent}$\{ (\{\Coin=\Heads \}, 0.49), $\\
\hspace*{\axiomsMedIndent}$( \{ \Coin = \Tails \}, 0.49), $\\
\hspace*{\axiomsMedIndent}$( \emptyset , 0.02)\} $\\

\noindent
\hspace*{\axiomsIndent}$ \Toss \performedAt 1$ \hfill (C4)\\

\noindent
where (C1) is a v-proposition, (C2) is an i-proposition, (C3) is a c-proposition and (C4) is a p-proposition.
\end{example}

\begin{example}\label{ex:antibioticdomain}
Scenario \ref{scenario:antibiotic} can be modeled using the following domain description $\mathcal{D}_A$:\\

\noindent
\hspace*{\axiomsIndent}$ \Bacteria \takesValues \{ \Weak, \Resistant, \Absent \} $\hfill(A1)\\

\noindent
\hspace*{\axiomsIndent}$ \Rash \takesValues \{ \Present, \Absent \} $\hfill(A2)\\

\noindent
\hspace*{\axiomsIndent}$ \initiallyOneOf$ \hfill(A3) \\
\hspace*{\axiomsMedIndent}$\{ (\{\Bacteria=\Weak, \Rash=\Present\}, 9/10),$\\
\hspace*{\axiomsMedIndent}$(\{\Bacteria=\Absent, \Rash=\Present\}, 1/10)\}$\\

\noindent
\hspace*{\axiomsIndent}$ \TakesMedicine \wedge \Bacteria=\Weak$ \hfill (A4)\\
\hspace*{\axiomsMedIndent}$\causesOneOf$\\
\hspace*{\axiomsBigIndent}$\{ (\{\Bacteria=\Absent, \Rash=\Absent\}, 7/10), $\\
\hspace*{\axiomsBigIndent}$(\{\Bacteria=\Resistant, \Rash=\Absent\}, 1/10), $\\
\hspace*{\axiomsBigIndent}$(\{\Bacteria=\Resistant\}, 2/10)\} $\\

\noindent
\hspace*{\axiomsIndent}$ \TakesMedicine \wedge \Bacteria=\Resistant$ \hfill (A5)\\
\hspace*{\axiomsMedIndent}$\causesOneOf$\\
\hspace*{\axiomsBigIndent}$\{ (\{\Bacteria=\Absent, \Rash=\Absent\}, 1/13), $\\
\hspace*{\axiomsBigIndent}$( \emptyset , 12/13)\} $\\

\noindent
\hspace*{\axiomsIndent}$ \TakesMedicine \performedAt 1$ \hfill (A6)\\

\noindent
\hspace*{\axiomsIndent}$ \TakesMedicine \performedAt 3$ \hfill (A7)\\

\noindent
where (A1) and (A2) are v-propositions, (A3) is an i-proposition, (A4) and (A5) are c-propositions, (A6) and (A7) are p-propositions.
\end{example}

\begin{example}\label{ex:keysdomain}
	Scenario \ref{scenario:keys} can be modeled using the following domain description $\mathcal{D}_K$:\\
	
	\noindent
	\hspace*{\axiomsIndent}$ \HasKeys \takesValues \{ \True, \False \} $\hfill(K1)\\
	
	\noindent
	\hspace*{\axiomsIndent}$ \LockedOut \takesValues \{ \True, \False \} $\hfill(K2)\\
	
	\noindent
	\hspace*{\axiomsIndent}$ \Location \takesValues \{ \Inside, \Outside \} $\hfill(K3)\\
	
	\noindent
	\hspace*{\axiomsIndent}$ \initiallyOneOf \{ (\{ \neg \HasKeys, \neg \LockedOut, \lit{\Location}{\Inside} \}, 1) \} $\hfill(K4)\\
	
	\noindent
	\hspace*{\axiomsIndent}$ \GoOut \wedge \neg \HasKeys \wedge \lit{\Location}{\Inside} $\hfill(K5)\\
	\hspace*{\axiomsMedIndent}$\causesOneOf$\\
	\hspace*{\axiomsBigIndent}$\{ (\{ \LockedOut, \lit{\Location}{\Outside} \}, 1) \} $\\
	
	\noindent
	\hspace*{\axiomsIndent}$ \GoOut \wedge \HasKeys \wedge \lit{\Location}{\Inside} $\hfill(K6)\\
	\hspace*{\axiomsMedIndent}$\causesOneOf \{ (\{ \lit{\Location}{\Outside} \}, 1) \} $\\
	
	\noindent
	\hspace*{\axiomsIndent}$ \PickupKeys \wedge \lit{\Location}{\Inside} $\hfill(K7)\\
	\hspace*{\axiomsMedIndent}$\causesOneOf \{ (\{ \HasKeys \}, 1) \} $\\
	
	\noindent
	\hspace*{\axiomsIndent}$ \PickupKeys \performedAt \text{7:30 AM} \withProb 0.99$\hfill(K8)\\	
	
	\noindent
	\hspace*{\axiomsIndent}$ \GoOut \performedAt \text{7:40 AM} $\hfill(K9)
\end{example}
Finally, we introduce h-propositions, whose role is that of being entailed by domain descriptions:
\begin{definition}[h-proposition]
	An \emph{h-proposition} has the form
	\begin{equation}\label{eq:hProposition}
	\varphi \holdsWithProb P.
	\end{equation}
	for some i-formula $\varphi$.
\end{definition}

For example, we will show in the following sections the formal sense in which $ \mathcal{D}_C $ entails the h-proposition ``$ \ifor{ \lit{\Coin}{\Heads} }{2} \holdsWithProb 0.51 $''.
\subsection{Semantics}
For the remainder of this paper, $ \mathcal{D} $ is an arbitrary domain description.

\begin{definition}[Worlds]\label{def:worlds} A \emph{world} is a function $W:\mathcal{I}\rightarrow \mathcal{S}$. The set of all worlds is denoted by $\mathcal{W}$.
\end{definition}

\begin{notation} In the following, we will use $W, W', W'', W_1, W_2, \dots$ to denote worlds.
\end{notation}

\begin{definition}[Satisfaction of an i-formula, Logical Consequence for i-formulas] Given a world $W$ and a literal $L$, $W$ \emph{satisfies an i-formula} \ifor{L}{I}, written $W \mmodels \ifor{L}{I}$, iff $L\in W(I)$\footnote{The symbols $\mmodels$ and $\notmmodels$ should not be confused with $\models$ and $\not\models$ which we use for the classical propositional entailment}. Otherwise we write $W\notmmodels \ifor{L}{I}$. The definition of $\mmodels$ is recursively extended for arbitrary i-formulas as follows: if $\varphi$ and $\varphi'$ are i-formulas, we write $W\mmodels \varphi \wedge \varphi'$ iff $W\mmodels \varphi$ and $W\mmodels '$, and $W\mmodels \neg \varphi$ iff $W\notmmodels \varphi$. $\vee$ and $\rightarrow$ are taken as shorthand in the usual way. Given a (possibly empty) set $\Delta$ of i-formulas, we write $W\mmodels \Delta$ iff $W \mmodels \psi$ for all $\psi \in \Delta$. Given an i-formula $\varphi$ and a set $\Delta$ of i-formulas we write $\Delta \mmodels \varphi$ if for all $W\in \mathcal{W}$ such that $W\mmodels \Delta$, $W\mmodels \varphi$ also holds. For two i-formulas $\psi$ and $\varphi$, we use $\psi \mmodels \varphi$ as a shorthand for $\{ \psi \} \mmodels \varphi$, and $\mmodels \varphi$ as a shorthand for $\emptyset \mmodels \varphi$.
\end{definition}

\begin{example}\label{ex:worlds}
	Three worlds for Scenario 1 can be specified as follows:\\
	
	\noindent
	\hspace*{\axiomsIndent}$W_1(0)= \{ \Coin=\Heads, \, \Toss=\False \}$,\\
	\hspace*{\axiomsIndent}$W_1(1)= \{ \Coin=\Heads, \, \Toss=\True \}$,\\
	\hspace*{\axiomsIndent}$W_1(I)= \{ \Coin=\Tails, \, \Toss=\False \}$ for all $I\geq2$.\\
	
	\noindent
	\hspace*{\axiomsIndent}$W_2(0)= \{ \Coin=\Tails, \, \Toss=\False \}$,\\
	\hspace*{\axiomsIndent}$W_2(1)= \{ \Coin=\Heads, \, \Toss=\False \}$,\\
	\hspace*{\axiomsIndent}$W_2(I)= \{ \Coin=\Tails, \, \Toss=\True \}$ for all $I\geq2$.\\
	
	\noindent
	\hspace*{\axiomsIndent}$W_3(0)= \{ \Coin=\Heads, \, \Toss=\False \}$,\\
	\hspace*{\axiomsIndent}$W_3(1)= \{ \Coin=\Heads, \, \Toss=\True \}$,\\
	\hspace*{\axiomsIndent}$W_3(I)= \{ \Coin=\Heads, \, \Toss=\False \}$ for all $I\geq2$.\\
	
	Intuitively, $W_1$ and $W_3$ match the domain description in Example \ref{ex:coindomain} as they represent a coherent history of what could have happened in Scenario \ref{scenario:coin}, whereas $W_2$ does not (e.g., changes occur when no action is performed, an infinite number of actions is being performed, etc\dots). This intuition will be made precise in what follows.
	
	Since worlds are functions from instants to states, they can conveniently be depicted as timelines as follows:\\
	
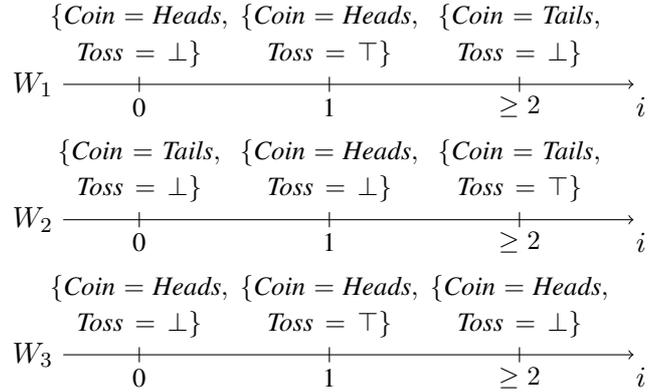
	\begin{center}
	\noindent\begin{tikzpicture}
	\draw[<-] (5.5,0) -- (-2,0) node[anchor=east] {$ W_1 $};
	\node at (5.6,-0.3) {$i$};
	
	\draw (-1,-0.1) -- (-1,0.1) node[anchor=south,text width=2.5cm,align=center] {\small$\{ \Coin=\Heads,$ $\Toss=\False \}$};
	\draw (1.5,-0.1) -- (1.5,0.1) node[anchor=south,text width=2.5cm,align=center] {\small$\{ \Coin=\Heads,$ $\Toss=\True \}$};
	\draw (4,-0.1) -- (4,0.1) node[anchor=south,text width=2.5cm,align=center] {\small$\{ \Coin=\Tails,$ $\Toss=\False \}$};

	\node at (-1,-0.3) {{\small {0}}};
	\node at (1.5,-0.3) {{\small {1}}};
	\node at (4,-0.3) {{\small {$\geq2$}}};
	\end{tikzpicture}\\
	
	\noindent\begin{tikzpicture}
	\draw[<-] (5.5,0) -- (-2,0) node[anchor=east] {$ W_2 $};
	\node at (5.6,-0.3) {$i$};
	
	\draw (-1,-0.1) -- (-1,0.1) node[anchor=south,text width=2.5cm,align=center] {\small$\{ \Coin=\Tails,$ $\Toss=\False \}$};
	\draw (1.5,-0.1) -- (1.5,0.1) node[anchor=south,text width=2.5cm,align=center] {\small$\{ \Coin=\Heads,$ $\Toss=\False \}$};
	\draw (4,-0.1) -- (4,0.1) node[anchor=south,text width=2.5cm,align=center] {\small$\{ \Coin=\Tails,$ $\Toss=\True \}$};
	
	\node at (-1,-0.3) {{\small {0}}};
	\node at (1.5,-0.3) {{\small {1}}};
	\node at (4,-0.3) {{\small {$\geq2$}}};
	\end{tikzpicture}\\
	
	\noindent\begin{tikzpicture}
	\draw[<-] (5.5,0) -- (-2,0) node[anchor=east] {$ W_3 $};
	\node at (5.6,-0.3) {$i$};
	
	\draw (-1,-0.1) -- (-1,0.1) node[anchor=south,text width=2.5cm,align=center] {\small$\{ \Coin=\Heads,$ $\Toss=\False \}$};
	\draw (1.5,-0.1) -- (1.5,0.1) node[anchor=south,text width=2.5cm,align=center] {\small$\{ \Coin=\Heads,$ $\Toss=\True \}$};
	\draw (4,-0.1) -- (4,0.1) node[anchor=south,text width=2.5cm,align=center] {\small$\{ \Coin=\Heads,$ $\Toss=\False \}$};
	
	\node at (-1,-0.3) {{\small {0}}};
	\node at (1.5,-0.3) {{\small {1}}};
	\node at (4,-0.3) {{\small {$\geq2$}}};
	\end{tikzpicture}
	\end{center}
\end{example}

\begin{definition}[Closed World Assumption for Actions]\label{def:cwa} A world $W$ is said to satisfy the \emph{closed world assumption for actions} (or \emph{CWA for actions}, for short) \emph{w.r.t. $\mathcal{D}$} if it satisfies the following condition: for all $A \in \mathcal{A}$ and $I\in \mathcal{I}$, if $W \mmodels \ifor{A}{I}$ then there exists some $P^+ \in (0,1]$ such that ``$A \performedAt I \withProb P^+$'' is in $ \mathcal{D} $. Furthermore, if for some $A \in \mathcal{A} $ and $I\in \mathcal{I}$ the p-proposition ``$A \performedAt I \withProb 1$'' is in $ \mathcal{D} $ then it must be the case that $W \mmodels \ifor{A}{I}$.
\end{definition}

\begin{example}\label{ex:cwa}
		Let $W_1$, $W_2$ and $W_3$ be the worlds in Example \ref{ex:worlds}, and let $\mathcal{D}_C$ be the domain description in Example \ref{ex:coindomain}. World $W_1$ satisfies CWA for actions w.r.t. $ \mathcal{D}_C $ as ${\Toss \in W_1 (I)}$ if and only if $I=1$, which is consistent with (C4) being the only p-proposition in $\mathcal{D}_C$. CWA is not satisfied by $W_2$ as $ \neg \Toss \in W_2 (1) $, i.e. $W_2 \mmodels \ifor{\neg \Toss}{1}$, but this is not consistent with (C4). $W_3$ satisfies CWA for actions for the same reason as $W_1$.
\end{example}

\begin{definition}[Cause Occurrence]\label{def:causeOccurrence} Let $\theta$ be the body of a c-proposition $c$ in a domain description $\mathcal{D}$ and $I\in \mathcal{I}$. If $W\mmodels \ifor{\theta}{I}$ then we say that that \emph{a cause occurs at instant $I$ in $W$ w.r.t. to $\mathcal{D}$}, and that \emph{the c-proposition $c$ is activated at $I$ in $W$ w.r.t. $\mathcal{D}$}. We write $\occ_{\mathcal{D}} (W)$ for the set $\{ I\in \mathcal{I} \mid \text{a cause occurs at $I$ in $W$} \}$. The function $cprop_\mathcal{D}$ with domain $\{ (W,I) \mid W\in \mathcal{W}, I \in occ_{\mathcal{D}}(W) \}$ is defined for instants $I$ in its domain as $cprop_\mathcal{D} (W,I)=c$ where $c$ is the (unique) c-proposition activated at $I$ in world $W$.
\end{definition}

\begin{example}
	Let $ \mathcal{D}_C $ be as in Example \ref{ex:coindomain} and $W_1$, $W_2$ and $W_3$ be as in Example \ref{ex:worlds}. Since $W_1 \mmodels \ifor{\Toss}{I} $ if and only if $I=1$ (and similarly for $W_3$), we derive that $\occ_{\mathcal{D}_C } (W_1) = \occ_{\mathcal{D}_C} (W_3) = \{ 1 \}$, with $\cprop_{\mathcal{D}_C} (W_1,1) = \cprop_{\mathcal{D}_C} (W_3,1) = \text{(C3)} $. For $W_2$, $\occ_{\mathcal{D}_C} (W_2)$ is defined as $\{ I \mid I\in\mathbb{N}, I\geq 2 \}$ with $\cprop_{\mathcal{D}_C} (W_2, I) = \text{(C3)}$ for $I\geq2$.
\end{example}

\begin{definition}[Initial Choice]\label{def:initialChoice}
	Let $\mathcal{D}$ be a domain description and the unique i-proposition in $\mathcal{D}$ be of the form \eqref{eq:iProposition}. Each $O_1,O_2,\dots,O_m$ is called an \emph{initial choice w.r.t. $\mathcal{D}$}.
\end{definition}

\begin{definition}[Effect Choice]\label{def:effectChoice}
	Let $W$ be a world and $\mathcal{D}$ a domain description. An \emph{effect choice for $W$} w.r.t. $\mathcal{D}$ is a function $\ec :  \occ_\mathcal{D} (W) \rightarrow \mathcal{O} $ such that for all instants $I \in \occ_\mathcal{D} (W)$, $\ec(I) \in \head(\cprop_\mathcal{D}(W,I)) $.
\end{definition}

\begin{example}\label{ex:icec}
	Let $ \mathcal{D}_C $ be as in Example \ref{ex:coindomain} and $W_1$, $W_2$ and $W_3$ be as in Example \ref{ex:worlds}. The only initial choice w.r.t. $ \mathcal{D}_C $ is $\ic_1 = (\{\Coin=\Heads\},1)$. The only effect choices for $W_1$ w.r.t. $\mathcal{D}_C$ are $\ec_1 (1) = (\{\Coin=\Tails\},49/100)$, $\ec_2 (1) = (\{\Coin=\Heads\},49/100)$ and $\ec_3 (1) = (\emptyset,2/100)$. Notice that since $occ_{\mathcal{D}_C}(W_1)=occ_{\mathcal{D}_C}(W_3)$, all the effect choices for $W_1$ are also effect choices for $W_3$.  There are an (uncountably) infinite number of effect choices for $W_2$ w.r.t. $\mathcal{D}_C$, each one mapping each instant $I\geq 2$ to $ (\{\Coin=\Heads\},49/100)$, $ (\{\Coin=\Tails\},49/100)$ or $ (\emptyset ,2/100)$. 
\end{example}

\begin{definition}[Initial Condition]\label{def:initialCondition}
	A world $W$ is said to \emph{satisfy the initial condition w.r.t. $\mathcal{D}$} if there exists an initial choice $\ic$ w.r.t. $\mathcal{D}$ such that $W(\initialInstant) \restriction \mathcal{F}= \chi (\ic) $. If a world $W$ satisfies the initial condition w.r.t. $\mathcal{D}$ for some initial choice $ \ic $, then we say that \emph{$W$ and $\ic$ are consistent with each other w.r.t. $ \mathcal{D} $}.
\end{definition}

\begin{example}\label{ex:ic}
	Let $ \mathcal{D}_C $ be as in Example \ref{ex:coindomain}, and $W_1$, $W_2$ and $W_3$ be as in \ref{ex:worlds}. Since $\ic_1 = (\{\Coin=\Heads\},1)$ is the only initial choice w.r.t. $\mathcal{D}_C$ as outlined in Example \ref{ex:icec}, $W_1$ and $W_3$ are consistent w.r.t. $\mathcal{D}_C$ with it, since $W_1 (0) \restriction \mathcal{F} = W_3 (0) \restriction \mathcal{F} = \chi (\ic_1)$. Therefore, $W_1$ and $W_3$ satisfy the initial condition w.r.t. $ \mathcal{D}_C $. Since $ W_2(0) \restriction \mathcal{F} \neq \chi(\ic_1) $, $W_2$ does not satisfy the initial condition.
\end{example}

\begin{definition}[Intervals] Given two instants $I$ and $I'$ such that $I \leq I'$, the intervals $[I,I']$, $[I,I')$, $(I,I']$ and $(I,I')$ are defined in the standard way w.r.t. the total order $\leq$. We also use $[I,+\infty)$ as shorthand for the set $\{ I' \mid I'\in \mathcal{I}, I'\geq I\}$, $(-\infty,I]$ as shorthand for $\{ I' \mid I'\in \mathcal{I}, I'\leq I\}$, $(I,+\infty)$ as a shorthand for $[I,\infty)\setminus \{ I \}$ and $(-\infty,I)$ as a shorthand for $(-\infty,I]\setminus \{ I \}$.
\end{definition}

\begin{definition}[Fluent State Update] Given a fluent state $\tilde{S}$ and a partial fluent state $\tilde{X}$, the \emph{update of $\tilde{S}$ w.r.t. $\tilde{X}$}, written $\update{\tilde{S}}{\tilde{X}}$, is the fluent state $(\tilde{S} \ominus \tilde{X}) \cup \tilde{X}$, where $\tilde{S} \ominus \tilde{X}$ is the partial fluent state formed by removing all fluent literals from $\tilde{S}$ of the form $F=V$ for some $F$ and $V'$ such that $F=V' \in \tilde{X}$. The operator $\oplus$ is left-associative, so e.g. $\tilde{S} \oplus \tilde{X} \oplus \tilde{X}'$ is understood as $(\update{(\update{\tilde{S}}{\tilde{X}})}{\tilde{X}'})$.
\end{definition}

\begin{definition}[Justified Change]\label{def:justifiedChange} A world $W$ is said to satisfy the \emph{justified change condition w.r.t. $\mathcal{D}$} if and only if there exists an effect choice $ \ec $ w.r.t. $\mathcal{D}$ such that for all instants $I$ and $I'$ with $I<I'$, $ \ec $ maps the possibly empty set of instants in $\occ_\mathcal{D} (W) \cap [I,I')=\{ I_1, \dots, I_n \}$ to $O_1, O_2, \dots, O_n$ respectively, where $I_1, \dots, I_n$ are ordered w.r.t. $\leq$, and
	\begin{equation}\label{eq:justifiedChange}
		W(I')\restriction \mathcal{F} = (W(I)\restriction \mathcal{F}) \oplus \chi(O_1) \oplus \chi(O_2) \oplus \dots \oplus \chi(O_n)
	\end{equation}
	If a world $W$ satisfies the justified change condition for some effect choice $\ec$, $W$ and $\ec$ are said to be \emph{consistent} with each other \emph{w.r.t. $\mathcal{D}$}.
\end{definition}

\begin{example}\label{ex:justifiedchange}
	Let $ \mathcal{D}_C $ be as in Example \ref{ex:coindomain}, $W_1$, $W_2$ be as in Example \ref{ex:worlds}, and $ec_1$ be defined as in Example \ref{ex:icec}.
	
	For any two instants $I$, $I'\in\mathbb{N}$ with $I<I'$, if $[I,I') \cap \occ_{\mathcal{D}_C} (W_1) = \emptyset $ then clearly $W_1(I')\restriction\mathcal{F} = W_1(I)\restriction \mathcal{F}$. Otherwise, if $[I,I') \cap \occ_{\mathcal{D}_C} (W_1) \neq \emptyset$, i.e. $[I,I') \cap \occ_{\mathcal{D}_C} (W_1) = \{ 1 \}$ then \eqref{eq:justifiedChange} holds as $W_1(I)\restriction\mathcal{F} \oplus \chi( ec_1 (1) ) = \{ \Coin=\Tails \} = W(I') \restriction \mathcal{F} $. So the justified change condition w.r.t. $ \mathcal{D}_C $ is satisfied by $W_1$.
	
	For $W$ to satisfy the justified change condition w.r.t. $ \mathcal{D}_C $, equation \eqref{eq:justifiedChange} would require $W_2(0) \restriction \mathcal{F} = W_2(1) \restriction \mathcal{F}$ (as $\occ_{\mathcal{D}_C} (W_2) \cap [1,2) = \emptyset$), but this is not the case. Hence, $W_2$ does not satisfy the justified change condition w.r.t. $ \mathcal{D}_C $.
\end{example}

\begin{definition}[Well-behaved Worlds] A world is said to be \emph{well-behaved w.r.t. $\mathcal{D}$} if it satisfies CWA for actions, the initial condition and the justified change condition w.r.t. $\mathcal{D}$. We denote the set of well-behaved worlds w.r.t. $ \mathcal{D} $ with $ \mathcal{W}_\mathcal{D} $.
\end{definition}

\begin{example}
	Let $ \mathcal{D}_C $ be as in Example \ref{ex:coindomain} and $W_1$, $W_2$ be as in Example \ref{ex:worlds}. $W_1$ is well-behaved as it satisfies CWA (see Example \ref{ex:cwa}), the initial condition (see Example \ref{ex:ic}) and the justified change condition (see Example \ref{ex:justifiedchange}) w.r.t. $ \mathcal{D}_C $. $W_2$ is not well-behaved as it fails to satisfy any of these conditions.
\end{example}




\begin{definition}[Candidate Trace, Trace]\label{def:trace}
	A \emph{candidate trace} is a function $\tr: \dom(\tr) \cup \{ \initialChoiceSymbol \} \rightarrow \mathcal{O} $ where $ \dom(\tr) \subseteq \mathcal{I} $ and $\initialChoiceSymbol$ is a new symbol such that $\initialChoiceSymbol \notin \mathcal{I}$. For readability, we will sometimes write $\langle \tr(\initialChoiceSymbol)@\initialChoiceSymbol, \tr(I_1)@I_1, \dots, \tr(I_m)@I_m \rangle$ where $\dom(\tr) = \{ I_1,\dots,I_m \}$ and the instants are ordered w.r.t. $\leq$.
	
	If $W$ is well-behaved w.r.t. $ \mathcal{D} $ and is consistent with the initial choice $\ic$ and the effect choice $\ec$ w.r.t. $ \mathcal{D} $, then $tr$ is said to be a \emph{trace of $W$ w.r.t. $ \mathcal{D} $} if $\dom(\tr) = \occ_\mathcal{D} (W)$ and $\tr(\initialChoiceSymbol) = \ic$ and for all $I\in \dom(\tr)$, $\tr(I) = \ec (I) $ and in this case we will sometimes write $\tr = (\ic,\ec)$.
	
	For any $W \in \mathcal{W}$, we write $TR_\mathcal{D}^W$ for the set of all traces of $W$ w.r.t. $ \mathcal{D} $, and notice that $TR^W_\mathcal{D} \neq \emptyset$ if and only if $W$ is well-behaved.
\end{definition}

A well-behaved world can have multiple traces, as shown in the following example.

\begin{example}\label{ex:trace}
Let $ \mathcal{D}_C $ be as in Example \ref{ex:coindomain}, $W_3$ be as in Example \ref{ex:worlds} and $\ic_1$, $\ec_2$, $\ec_3$ be as defined in Example \ref{ex:icec}. World $W_3$ has two distinct traces, $tr'_3=(ic_1,\ec_2)$ and $\tr_3''=(\ic_1,\ec_3)$, which disagree on the effect choice: in one case the robot manages to toss the coin producing $\Coin = \Heads$ as a result (i.e., $\tr_3'(1)=(\{\Coin=\Heads \}, 0.49)$) whereas in the other case the robot fails to grab the coin (i.e., $\tr_3''(1)=(\emptyset, 0.02)$) leaving $\Coin=\Heads$ to hold. These two traces are also the only traces of this world w.r.t. $ \mathcal{D_C} $.
\end{example}

However, for some candidate traces $ \tr $ there exists no well-behaved world $ W $ such that $ \tr $ is a trace of $ W $. We now generalise Definition \ref{def:trace} to domain descriptions:

\begin{definition}[Trace of a Domain Description]\label{def:traceOfADomainDescription} Given a candidate trace $\tr$, if there exists a well-behaved world $W$ w.r.t. $ \mathcal{D} $ such that $\tr$ is a trace of $ W $ w.r.t. $ \mathcal{D} $, then $tr$ is said to be \emph{a trace of $\mathcal{D}$}. 
\end{definition}

\begin{definition}[Evaluation of a Trace]\label{def:evaluationTrace} Let $tr$ be a candidate trace. The \emph{evaluation of $\tr$}, written $ \epsilon( \tr ) $, is defined as:
	\begin{equation}\label{eq:trace}
		\epsilon (\tr) = \pi( \tr(\initialChoiceSymbol) ) \cdot \prod_{ \mathclap{I\in \dom(\textit{tr}) }} \pi( \tr(I) )
	\end{equation}
\end{definition}

\begin{definition}[Evaluation of a Narrative] Given a p-proposition $ p $ of the form ``$ A \performedAt I \withProb P^+ $'', we define the \emph{evaluation of $p$ w.r.t. $ W $} as
\begin{equation}\label{eq:evalpproposition}
	\epsilon (p,W) = \begin{cases}
		P^+ & \text{if } W \mmodels \ifor{A}{I} \\
		1 - P^+ & \text{otherwise}
	\end{cases}
\end{equation}

For an action narrative $N$ (see Definition \ref{def:narrative}) we extend the previous definition to:
\begin{equation}
	\epsilon ( N, W ) = \prod_{\mathclap{p\in N}} \epsilon ( p, W ).
\end{equation}
and write $ \epsilon_\mathcal{D} (W) $ as a shorthand for $ \epsilon( \narr(\mathcal{D}),W ) $. Conventionally, $\epsilon(N,W)=1$ when $N=\emptyset$.
\end{definition}

\begin{definition}[$\lbrack 0,1 \rbrack$-interpretation] A \emph{[0,1]-interpretation} is a function from $\mathcal{W}$ to $[0,1]$.
\end{definition}

\begin{definition}[Model]\label{def:model} A \emph{model} of a domain description $\mathcal{D}$ is a $[0,1]$-interpretation $\model{\mathcal{D}}$ such that 
	\begin{enumerate}
		\item If $W\in\mathcal{W}$ is not well-behaved w.r.t. $ \mathcal{D} $,
		\begin{equation}
			\model{\mathcal{D}}{W} = 0,
		\end{equation}
		\item If $W\in\mathcal{W}$ is well-behaved w.r.t. $ \mathcal{D} $,
		\begin{equation}\label{eq:model}
			\model{\mathcal{D}}{W}= \epsilon_{\mathcal{D}} (W) \cdot \sum_{\mathclap{tr \in TR_\mathcal{D}^W} } \epsilon (\tr).
		\end{equation}
	\end{enumerate}
\end{definition}

\begin{example}\label{ex:modelofw3}
	Let $\mathcal{D}_C$ be as in Example \ref{ex:coindomain} and $W_3$ be as Example \ref{ex:worlds}. As discussed in Example \ref{ex:trace}, $W_3$ has exactly two traces $ \tr_3' = (ic_1,\ec_2) $ and $ \tr_{3}''= (\ic_1,\ec_3) $. Equations \eqref{eq:trace} and \eqref{eq:model} yield:
	\[ \model{\mathcal{D}_C}{3} = \epsilon ( \tr_{3}' ) + \epsilon ( \tr_{3}'' ) = 0.49 + 0.02 = 0.51 \]
\end{example}

\begin{proposition} A domain description $\mathcal{D}$ has a unique model.
	\begin{proof} This can be derived from Definition \ref{def:model} by considering that $\model{\mathcal{D}}{W}$ is calculated as a product of functions of the states of $W$.
	\end{proof}
\end{proposition}

\begin{definition}\label{def:Mstar} We extend the model $\model{\mathcal{D}}$ to a function $\modelStar{\mathcal{D}}: \Phi \rightarrow [0,1]$ over i-formulas in the following way: \[ \modelStar{\mathcal{D}}{\varphi} = \sum_{\mathclap{W \mmodels \varphi}} \model{\mathcal{D}}{W}.\]
\end{definition}

\begin{definition}[Entailment for Domain Descriptions] Given a domain description $\mathcal{D}$ and an i-formula $ \varphi $, we say that \emph{the h-proposition ``$\varphi \holdsWithProb P$'' is entailed by $\mathcal{D}$} iff $\modelStar{\mathcal{D}}{\varphi}=P$.
\end{definition}

\begin{example}
	For $\mathcal{D}_C$ as in Example \ref{ex:coindomain}, the only well-behaved world $W$ such that $W \mmodels \ifor{\lit{\Coin}{\Heads}}{2} $ is $W_3$. Definition \ref{def:Mstar} and Example \ref{ex:modelofw3} yield
    \[ \modelStar{\mathcal{D}_C}{\ifor{\lit{\Coin}{\Heads}}{2}} = \model{\mathcal{D}_C}{W_3} = 0.51. \]
The reader can verify that $\ifor{\lit{\Coin}{\Heads}}{0} $ yields
	\[ \modelStar{\mathcal{D}_C}{\ifor{\lit{\Coin}{\Heads}}{0}} = \model{\mathcal{D}_C}{W_1} + \model{\mathcal{D}_C}{W_3} = 1 \]

\noindent
and from this we can derive that $ \mathcal{D}_C $ entails the two following h-propositions:\\
	
	\noindent
	\hspace*{\axiomsIndent}$\ifor{\lit{\Coin}{\Heads}}{2} \holdsWithProb 0.51$,\\
	
	\noindent
	\hspace*{\axiomsIndent}$\ifor{\lit{\Coin}{\Heads}}{0} \holdsWithProb 1$.\\
\end{example}

\begin{example}
	Let $ \mathcal{D}_K $ be as in Example \ref{ex:keysdomain}. Assuming $\initialInstant = \text{7:30 AM}$, the only two well-behaved worlds w.r.t. $ \mathcal{D}_K $ are:
	
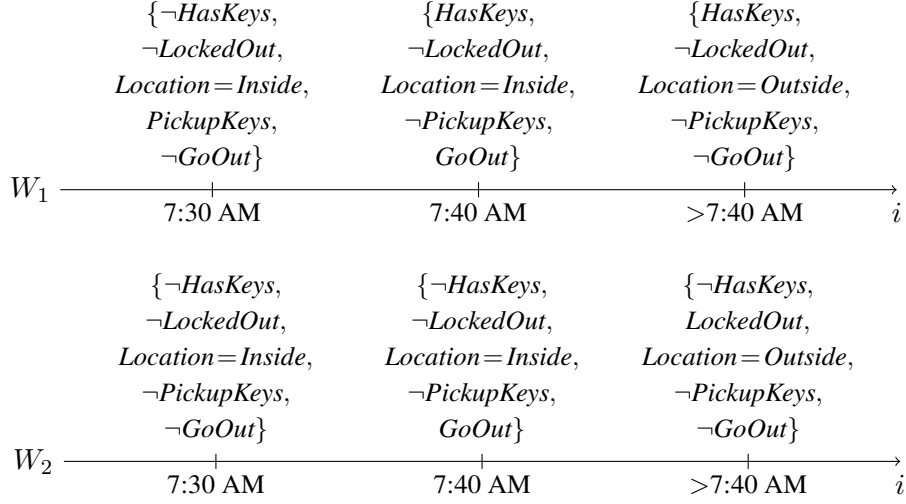
\begin{center}
	\noindent\begin{tikzpicture}
	\draw[<-] (9,0) -- (-2,0) node[anchor=east] {$ W_1 $};
	\node at (9,-0.3) {$i$};
	
	\draw (0,-0.1) -- (0,0.1) node[anchor=south,text width=3cm,align=center] {\small$\{ \neg\HasKeys,$ $\neg\LockedOut,$ $ \lit{\Location}{\Inside},$ $\PickupKeys,$ $\neg\GoOut \}$};
	\draw (3.5,-0.1) -- (3.5,0.1) node[anchor=south,text width=3cm,align=center] {\small$\{ \HasKeys,$ $\neg\LockedOut,$ $ \lit{\Location}{\Inside},$ $\neg\PickupKeys,$ $\GoOut \}$};
	\draw (7,-0.1) -- (7,0.1) node[anchor=south,text width=3cm,align=center] {\small$\{ \HasKeys,$ $\neg\LockedOut,$ $ \lit{\Location}{\Outside},$ $\neg\PickupKeys,$ $\neg\GoOut \}$};
	
	\node at (0,-0.3) {{\small {7:30 AM}}};
	\node at (3.5,-0.3) {{\small {7:40 AM}}};
		\node at (7,-0.3) {{\small {$>$7:40 AM}}};
	\end{tikzpicture}
	\vspace{1em}
	
	\noindent\begin{tikzpicture}
		\draw[<-] (9,0) -- (-2,0) node[anchor=east] {$ W_2 $};
		\node at (9,-0.3) {$i$};
		
		\draw (0,-0.1) -- (0,0.1) node[anchor=south,text width=3cm,align=center] {\small$\{ \neg\HasKeys,$ $\neg\LockedOut,$ $ \lit{\Location}{\Inside},$ $\neg\PickupKeys,$ $\neg\GoOut \}$};
		\draw (3.5,-0.1) -- (3.5,0.1) node[anchor=south,text width=3cm,align=center] {\small$\{ \neg\HasKeys,$ $\neg\LockedOut,$ $ \lit{\Location}{\Inside},$ $\neg\PickupKeys,$ $\GoOut \}$};
		\draw (7,-0.1) -- (7,0.1) node[anchor=south,text width=3cm,align=center] {\small$\{ \neg\HasKeys,$ $\LockedOut,$ $ \lit{\Location}{\Outside},$ $\neg\PickupKeys,$ $\neg\GoOut \}$};
		
		\node at (0,-0.3) {{\small {7:30 AM}}};
		\node at (3.5,-0.3) {{\small {7:40 AM}}};
		\node at (7,-0.3) {{\small {$>$7:40 AM}}};
	\end{tikzpicture}
	\end{center}

	Since both actions $ \GoOut $ and $ \PickupKeys $ have definite effects (i.e., outcomes in the head of the corresponding c-propositions have probability equal to $ 1 $), the only significant factors in the calculation of $ \model{\mathcal{D}_K} $ are those given by the evaluation of the action narrative:
	
	\[ \model{\mathcal{D}_K}{W_1} = \epsilon_{\mathcal{D}_K} (W_1) = 0.99 \]
	
	\[ \model{\mathcal{D}_K}{W_2} = \epsilon_{\mathcal{D}_K} (W_2) = 0.01 \]
	implying that $ \mathcal{D}_K $ entails the following propositions:\\
	
	\noindent
	\hspace*{\axiomsIndent}$\ifor{\LockedOut}{\text{9AM}} \holdsWithProb 0.01$,\\

	\noindent	
	\hspace*{\axiomsIndent}$\ifor{\HasKeys}{\text{9AM}} \holdsWithProb 0.99$.
\end{example}

\subsection{Properties of a model}

We now introduce the concept of a \emph{probability function}, adapted from \cite{paris2006uncertain}:

\begin{definition}[Probability Function, Conditional Probability]\label{def:probabilityFunction} A \emph{probability function (over i-formulas)} is a function $p:\Phi \rightarrow [0,1]$ such that:
	\begin{enumerate}
		\item \label{item:probabilityFunction1} if $\mmodels \varphi$, then $p(\varphi)=1$, 
		\item \label{item:probabilityFunction2} if $\varphi \mmodels \neg \psi$ for two i-formulas $\varphi$ and $\psi$, then $p(\varphi\vee \psi)=p(\varphi)+p(\psi)$.
	\end{enumerate}
	The associated \emph{conditional probability} of $\varphi$ given $\psi$ is defined as
	\begin{equation}\label{eq:conditionalprobability}
	p(\varphi \mid \psi)=\frac{p(\varphi\wedge\psi)}{p(\psi)}
	\end{equation}
	for $p(\psi)\neq 0$.
\end{definition}



We will show that $\modelStar{\mathcal{D}}$ is a probability function. To prove this, first we need to introduce some auxiliary definitions:

\begin{definition}[Restricted Domain Description]\label{def:restricteddomaindescription} If $\mathcal{D}$ is a domain description, we denote by $ \mathcal{D}_{\leq I} $ the domain description obtained from $\mathcal{D}$ by removing all the p-propositions occurring at instants $ > I $, and similarly we denote by $ \mathcal{D}_{< I}$ the domain description obtained from $\mathcal{D}$ by removing all the p-propositions occurring at instants $ \geq I $. Finally, we denote by $\mathcal{D}_{\emptyset}$ the domain description obtained from $\mathcal{D}$ by removing all p-propositions, i.e. $\mathcal{D}_{<\initialInstant}$.
\end{definition}

\begin{definition}[Fluent-indistinguishability, Indistinguishability]\label{def:fluentindistinguishability} A world $W$ is said to be \emph{fluent-indistinguishable from $W'$ up to an instant $I$} if and only if $W(I') \restriction \mathcal{F} =W'(I') \restriction \mathcal{F}$ for all instants $I'$ such that $I' \leq I$. $W$ is said to be \emph{indistinguishable from $W'$ up to an instant $I$} if and only if it is fluent indistinguishable from $ W' $ up to $ I $ and if for all $I'<I$ it also satisfies $ A\in W(I') $ if and only if $ A\in W'(I') $.
\end{definition}

In the following example, we illustrate the two concepts of restricted domain description and indistinguishability:

\begin{example}\label{ex:restrictionandfluentindistinguishability} Let $ \mathcal{D}' $ be the domain description obtained from $ \mathcal{D}_C $ as in Example \ref{ex:coindomain} by adding the following p-proposition:\\

\noindent
\hspace*{\axiomsIndent}$ \Toss \performedAt 2$\hfill(C5)\\

\noindent
and consider the following well-behaved world w.r.t. $ \mathcal{D}' $:\\

\noindent
\hspace*{\axiomsIndent}$W'(0) = \{ \Coin = \Heads, \neg \Toss \}$,\\
\hspace*{\axiomsIndent}$W'(1) = W'(2) = \{ \Coin = \Heads, \Toss \}$,\\
\hspace*{\axiomsIndent}$W'(I) =\{ \Coin = \Tails, \neg \Toss \}$ for all $I > 2$\\

\noindent
$W'$ has exactly two traces $tr' = \langle (\{\Coin=\Heads\},1)@\initialChoiceSymbol, (\{\Coin=\Heads\},0.49)@1, (\{\Coin=\Tails\},0.49)@2 \rangle$ and $tr''= \langle (\{\Coin=\Heads\},1)@\initialChoiceSymbol, (\emptyset,0.02)@1, (\{\Coin=\Heads\},0.49)@2 \rangle$.

Consider $\mathcal{D}'_{<2} $ and notice that it coincides with $\mathcal{D}_C$ as in the previous examples. There is a unique well-behaved world w.r.t. $\mathcal{D}_C$ that is indistinguishable from $W'$ up to $2$, and this world is $W_3$ as in Example \ref{ex:worlds}.
\end{example}

\begin{definition}[Transition Set, Transition Function]\label{def:transition} Given a domain description $ \mathcal{D} $, a state $ S $ and a fluent state $ \fluentState{S}' $, the \emph{transition set} $\tset_{\mathcal{D}}( S, \fluentState{S}' )$ is defined as follows: if $ \mathcal{D} $ contains a (unique) c-proposition $c$ such that $S$ Herbrand-entails $\body(c)$, then $\tset_{\mathcal{D}}( S, \fluentState{S}' ) = \{ O \in \head(c) \mid (S\restriction\mathcal{F}) \oplus \chi(O)=\fluentState{S}' \}$ if there is no such c-proposition and $S\restriction\mathcal{F} = \fluentState{S}'$ then $\tset( S, \fluentState{S}' ) = \{ (\emptyset,1) \}$; otherwise, $\tset_{\mathcal{D}}( S, \fluentState{S}' ) = \emptyset$.
	
The \emph{transition function} for a domain description $ \mathcal{D} $ is the function $ t_\mathcal{D} : \mathcal{S} \times \fluentState{\mathcal{S}} \rightarrow [0,1] $ defined by $t_\mathcal{D} (S,\fluentState{S}') = \weight(\tset_\mathcal{D}(S,\fluentState{S}')) $ (recall Definition \ref{def:weightOfASetOfOutcomes} for the meaning of $ \weight $ in this case).
\end{definition}

Informally, the transition function gives the probability of moving from state $ S $ to the fluent state $ \fluentState{S}' $ within $ \mathcal{D} $, independently of its particular narrative.

The transition function for the coin toss example can be visualised as in Figure \ref{fig:coinTransition}, where the nodes represent fluent states (in this case we have two nodes $H$ and $T$ standing for the fluent states $\{ \Coin=\Heads \}$ and $\{ \Coin=\Tails \}$ respectively), and if $p=t_\mathcal{D} (S,\tilde{S}')$ for some state $S$ and some fluent state $ \fluentState{S}' $, then there is an arrow from a node representing $S \restriction \mathcal{F} $ to a node representing $\tilde{S}'$ which is labelled $S\restriction\mathcal{A}, p$. The arrow is omitted in some trivial cases (for instance when the set of actions is empty).

\begin{figure}
\begin{center}
	\begin{tikzpicture}[>=latex',shorten >=2pt,node distance=6cm,on grid,auto]
	
	\node[state] (H) {$H$};
	\node[state] (T) [right=of H] {$T$};
	\path[->] (H) edge [loop above,looseness=6] node [anchor=south] {$\{\Toss\}$, 0.51} (H);
	\path[->] (T) edge [loop above,looseness=6] node [anchor=south] {$\{\Toss\}$, 0.51} (T);
	\path[->] (H) edge [bend right=35] node [anchor=north] {$\{ \Toss \}$, 0.49} (T);
	\path[->] (T) edge [bend right=35] node [anchor=south] {$\{ \Toss \}$, 0.49} (H);
	\end{tikzpicture}
\end{center}
\caption[prova]{Transition function for the Coin Toss domain.}\label{fig:coinTransition}
\end{figure}
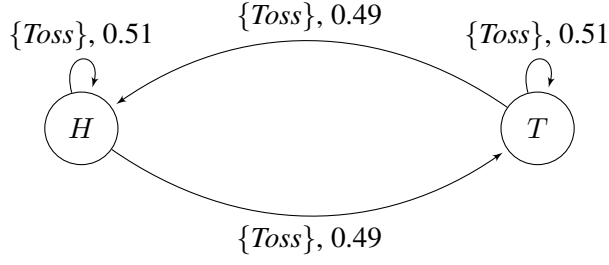

Similarly, the transition function for the antibiotic domain can be pictured as in Figure \ref{fig:antibioticTransition}.

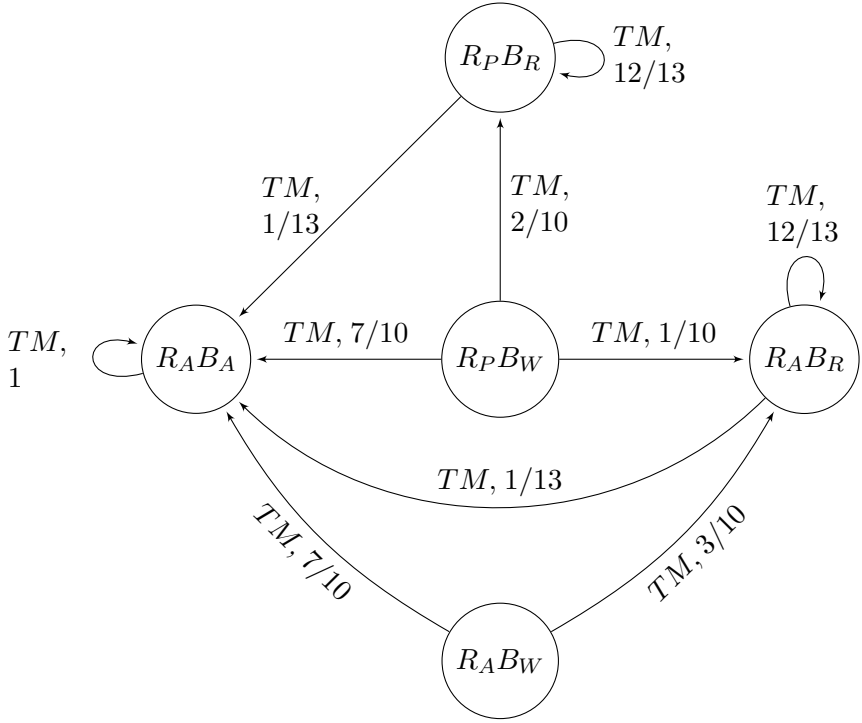
\begin{figure}
	\begin{tikzpicture}[>=latex',shorten >=2pt,node distance=4cm,on grid,auto]

	\node[state] (RPBR) {$R_P B_R$};
	\node[state] (RPBW) [below=of RPBR] {$R_P B_W$};
	\node[state] (RABA) [left=of RPBW] {$R_A B_A$};
	\node[state] (RABR) [right=of RPBW] {$R_A B_R$};
	\node[state] (RABW) [below=of RPBW] {$R_A B_W$};
	\path[->] (RPBR) edge [loop right,looseness=6,text width=1cm] node [anchor=west] {$TM,$ $12/13$} (RPBR);
	\path[->] (RABR) edge [loop above,looseness=6,text width=1cm] node [anchor=south] {$TM,$ $12/13$} (RABR);
	\path[->] (RABA) edge [loop left,looseness=6,text width=1cm] node [anchor=east] {$TM,$ $1$} (RABA);
	\path[->] (RPBR) edge node [anchor=east,text width=1cm] {$TM,$ $1/13$} (RABA);
	\path[->] (RPBW) edge node [anchor=south] {$TM,$ $7/10$} (RABA);
	\path[->] (RPBW) edge node [anchor=south] {$TM,$ $1/10$} (RABR);
	\path[->] (RPBW) edge node [anchor=west,text width=1cm] {$TM,$ $2/10$} (RPBR);
	\path[->] (RABR) edge [bend left=45] node [anchor=south] {$TM,$ $1/13$} (RABA);
	\path[->] (RABW) edge [bend left=15] node [anchor=north, rotate=-45] {$TM,$ $7/10$} (RABA);
	\path[->] (RABW) edge [bend right=15] node [anchor=north, rotate=45] {$TM,$ $3/10$} (RABR);
	\end{tikzpicture}
	\caption{Transition function for the Antibiotic domain.}\label{fig:antibioticTransition}
\end{figure}
where $R_P B_R$ is the fluent state $\{ \Rash = \Present, \Bacteria=\Resistant \}$, $R_P B_W$ is the fluent state $\{ \Rash = \Present, \Bacteria=\Weak \}$, $ R_A B_A $ is the fluent state $\{ \Rash = \Absent, \Bacteria=\Absent \}$, $ R_A B_R $ is the fluent state $\{ \Rash = \Absent, \Bacteria=\Resistant \}$, and $TM=\{ \TakesMedicine \}$.

The transition function can be conveniently used to express $\model{\mathcal{D}}{W}$ in terms of the model of a well-behaved world w.r.t. an appropriately restricted domain description:

\begin{proposition}\label{prop:decomposingw}
Let $\mathcal{D}$ be an arbitrary domain description and $W$ be a world such that $\occ_\mathcal{D} (W) = \{ I_1, \dots, I_n \} \neq \emptyset$ where $I_1, \dots, I_n$ are ordered w.r.t. $\leq$, and let $c$ be the c-proposition activated in $W$ at $I_{n}$ w.r.t. $ \mathcal{D} $. Then $W$ is well-behaved w.r.t. $ \mathcal{D} $ if and only if (i) there exists a unique world $W'$ well-behaved w.r.t. $\mathcal{D}_{<I_n}$ which is indistinguishable from $W$ up to $I_n$, (ii) for all $I > I_n$, $W(I)\restriction\mathcal{F} = \fluentState{S}^{W}_{>I_n}$ where $ \fluentState{S}^{W}_{>I_n} = (W(I_n)\restriction\mathcal{F}) \oplus \chi( O ) $ for some outcome $O \in \head(c)$ , and (iii) $W$ satisfies CWA for actions w.r.t. $ \mathcal{D} $.
	
	Furthermore, for $\fluentState{S}^{W}_{>I_n}$ the unique fluent state taken by $W$ at instants $I>I_n$:
\begin{equation}\label{eq:recursiveModel}
	\model{\mathcal{D}}{W} = \frac{\epsilon_\mathcal{D} ( W )}{\epsilon_{\mathcal{D}_{<I_n}} ( W' )} \cdot \model{\mathcal{D}_{<I_n}}{W'} \cdot t_\mathcal{D} ( W(I_n), \fluentState{S}^{W}_{>I_n} )
\end{equation}
\end{proposition}

\begin{proof} \textbf{``Only if'' subproof.} Let $W$ be well-behaved w.r.t. $\mathcal{D}$. Let $\tr = \langle \tr(\initialChoiceSymbol)@\initialChoiceSymbol, \tr(I_1)@I_1, \dots, \tr(I_n)@I_n \rangle$ be an arbitrary trace of $W$ w.r.t. $\mathcal{D}$ and consider the candidate trace $\tr' = \langle \tr(\initialChoiceSymbol)@\initialChoiceSymbol, \tr(I_1)@I_1, \dots, \tr(I_{n-1})@I_{n-1} \rangle$.

Since W is well-behaved w.r.t. $ \mathcal{D} $, since $\mathcal{D}$ and $\mathcal{D}_{< I_n}$ differ only by one p-proposition occurring at $I_{n}$, and since $\tr'$ does not mention any instant strictly greater than $I_{n-1}$, it is possible to construct a world $W'$ which has trace $\tr'$ w.r.t. $\mathcal{D}_{< I_n}$ and which is fluent-indistinguishable from $W$ up to instant $I_{n}$ by simply considering that $W'(\initialInstant) = \chi(\tr'(\initialChoiceSymbol)) = W(\initialInstant)$ makes the Initial Condition satisfied w.r.t. $ \mathcal{D}_{< I_n} $ as both $ \mathcal{D} $ and $ \mathcal{D}_{< I_n} $ share the same i-proposition, and a similar argument applies to the Justified Change Condition w.r.t. $ \mathcal{D}_{< I_n} $. Well-behavedness w.r.t. $ \mathcal{D} $ and $ \mathcal{D}_{< I_n} $ guarantees that for all instants $I \leq I_n$, $W(I)\restriction \mathcal{F} = \tr(\initialChoiceSymbol) \oplus \dots \oplus \tr(I_i) = \tr'(\initialChoiceSymbol) \oplus \dots \oplus \tr'(I_i) = W'(I)\restriction \mathcal{F} $ for some $i<n$, hence $W$ is fluent-indistinguishable from $ W' $ up to $I_n$. Such a $W'$ might not be unique, but if we choose $W'$ as to satisfy $A \in W(I) \Leftrightarrow A\in W'(I)$ for all $I<I_n$ then the uniqueness of $W'$ is guaranteed. Then, (i) holds. Since $ W $ is well-behaved w.r.t. $ \mathcal{D} $ and $I_n$ is the greatest element in $\occ_\mathcal{D}$, $\chi(\tr(I_n)) = O$ for some $O \in \head(c)$ and Justified Change implies $W(I) \restriction \mathcal{F} = (W(I_n) \restriction \mathcal{F}) \oplus \chi( \tr(I_n) ) $ for all $I > I_n$, and if we let $\fluentState{S}^{W}_{>I_n}$ be such unique fluent state (ii) is also satisfied. Finally, (iii) holds by definition of well-behavedness w.r.t. $ \mathcal{D} $.

\textbf{``If'' subproof.} Let $ W' $ be a well-behaved world w.r.t. $ \mathcal{D}_{<I_n} $ and let $\occ_{\mathcal{D}_{<I_n}} (W') = \{ I_1, \dots, I_{n-1} \}$. Let $\tr' = \langle \tr(\initialChoiceSymbol)@\initialChoiceSymbol, \tr(I_1)@I_1, \dots, \tr(I_{n-1})@I_{n-1} \rangle$ be a trace of $W'$ w.r.t. $ \mathcal{D}_{I_n} $ and construct the candidate trace $\tr = \langle \tr(\initialChoiceSymbol)@\initialChoiceSymbol, \tr(I_1)@I_1, \dots, O@I_n \rangle$ for the outcome $O \in \head(c)$ such that $(W(I)\restriction\mathcal{F}) = (W(I_n)\restriction\mathcal{F}) \oplus \chi(O)$ for all $I>I_n$. Since $W'$ is well-behaved w.r.t. $ \mathcal{D}_{<I_n} $ and indistinguishable from $W$ up to $I_n$ by hypothesis (i), we derive that $tr$ is a trace of $W'$ w.r.t. $ \mathcal{D} $ by noticing again that both $ \mathcal{D} $ and $ \mathcal{D}_{< I_n} $ share the same i-proposition, and a similarly arguments applies for the Justified Change Condition w.r.t. $ \mathcal{D} $ (also using hypothesis (ii)). Since $W$ also satisfies CWA for actions by hypothesis (iii), it is well-behaved.

\textbf{``Furthermore'' subproof.} Let $\fluentState{S}^{W}_{>I_n}$ and $c$ be as in the statement of the proposition. The above proof implies that for any trace $\tr$ of $W$ w.r.t. $ \mathcal{D} $ this trace can be constructed from a trace $\tr'$ of $ W' $ by letting $\tr(I_i) = \tr'(I_i)$ for $i<I_n$ and $\tr(I_n) = \chi(O)$ for some $O \in \head(c)$ such that $\fluentState{S}^{W}_{>I_n} = (W(I_{n}) \restriction \mathcal{F}) \oplus \chi(O) $ (and notice that there is at least such an outcome $O$ since $W$ is well-behaved), i.e. for some $O \in \tset_{\mathcal{D}} ( W(I_n), \fluentState{S}^{W}_{>I_n} )$.

Definition \ref{def:model} now implies

\begin{align*}
	\model{\mathcal{D}}{W} &= \epsilon_{\mathcal{D}} (W) \cdot \sum_{\mathclap{tr \in TR_\mathcal{D}^W} } \epsilon (\tr) \\
		&= \epsilon_{\mathcal{D}_{< I_n}}(W') \cdot \frac{\epsilon_{\mathcal{D}}(W)}{\epsilon_{\mathcal{D}_{< I_n}}(W')} \cdot \pi( \tset_{\mathcal{D}} ( W(I_n), \fluentState{S}^{W}_{>I_n} ) ) \cdot \sum_{\mathclap{ tr' \in TR^{W'}_{\mathcal{D}_{<I_n}} }} \epsilon (\tr')\\
		&= \frac{\epsilon_{\mathcal{D}}(W)}{\epsilon_{\mathcal{D}_{< I_n}}(W')} \cdot t_\mathcal{D} ( W(I_n), \fluentState{S}^{W}_{>I_n} ) \cdot \left( \epsilon_{\mathcal{D}_{< I_n}}(W') \cdot \sum_{\mathclap{ tr' \in TR^{W'}_{\mathcal{D}_{<I_n}} }} \epsilon (\tr') \right)\\
		&= \frac{\epsilon_\mathcal{D} ( W )}{\epsilon_{\mathcal{D}_{<I_n}} ( W' )} \cdot t_\mathcal{D} ( W(I_n), \fluentState{S}^{W}_{>I_n} ) \cdot \model{\mathcal{D}_{<I_n}}{W'}.
\end{align*}
which is well defined since $\epsilon(N,W) > 0$ for any action narrative $N$ and world $W$.
\end{proof}

\begin{corollary}\label{cor:decomposingw} Let $ \mathcal{D} $ be any domain description and let $ I $ be any instant. Then $W$ is well-behaved w.r.t. $ \mathcal{D}_{\leq I} $ if and only if (i) there exists a unique world $W'$ well-behaved w.r.t. $\mathcal{D}_{<I}$ which is indistinguishable from $W$ up to $I$, (ii) for all $I' > I$ then $W(I')\restriction\mathcal{F} = \fluentState{S}^{W}_{>I_n}$ where $ \fluentState{S}^{W}_{>I_n} = (W(I)\restriction\mathcal{F}) \oplus \chi( O ) $ for some outcome $O \in \head(c)$ if $I \in \occ_{\mathcal{D}_{\leq I}}(W)$, and $\fluentState{S}^{W}_{>I_n}=W(I)\restriction\mathcal{F}$ otherwise, and (iii) $W$ satisfies CWA for actions w.r.t. $ \mathcal{D}_{\leq I} $.
	
	Furthermore, for $\fluentState{S}^{W}_{>I_n}$ the unique fluent state taken by $W$ at instants $I>I_n$:
\begin{equation}\label{eq:recursiveModelCorollary}
\model{\mathcal{D}_{\leq I}}{W} = \frac{\epsilon_{\mathcal{D}_{\leq I}} ( W )}{\epsilon_{\mathcal{D}_{<I}} ( W' )} \cdot \model{\mathcal{D}_{<I}}{W'} \cdot t_\mathcal{D} ( W(I), \fluentState{S}^{W}_{>I} )
\end{equation}
\end{corollary}

\begin{proof}
	If $I \in \occ_{\mathcal{D}_{\leq I}} (W)$ then the corollary follows directly from Proposition \ref{prop:decomposingw} since the domain description $ \mathcal{D}_{\leq I} $ satisfies all of its hypotheses. If $I \notin \occ_{\mathcal{D}_{\leq I}} (W)$ then no change of state occurs at $I$ in $W$, i.e. $W(I)\restriction \mathcal{F} = W(I')\restriction \mathcal{F} = \fluentState{S}^{W}_{>I}$ for any $I'>I$, and Equation \eqref{eq:recursiveModelCorollary} holds for $t_\mathcal{D} ( W(I), \fluentState{S}^{W}_{>I} ) = 1$.
\end{proof}

\begin{lemma}\label{lem:t} For any $ \mathcal{D} $ and any state $S$, \[ \sum_{\fluentState{S}' \in \fluentState{\mathcal{S}} } t_\mathcal{D} (S, \fluentState{S}') = 1. \]
\end{lemma}
\begin{proof}
	We prove this by cases:
	
	\textbf{Case 1.} If there is no c-proposition $c$ such that $S$ Herbrand-entails $\body(c)$, then it follows from Definition \ref{def:transition} that \[ \sum_{\mathclap{\fluentState{S}' \in \fluentState{\mathcal{S}}} } t_\mathcal{D}(S, \fluentState{S}') = t_\mathcal{D} (S, S\restriction\mathcal{F}) = \pi( (\emptyset, 1) ) = 1 \] which is what we want.
	
	\textbf{Case 2.} Let $c$ be the unique c-proposition $S$ Herbrand-entails $\body(c)$. Then, applying the definition of $t_\mathcal{D}$ from Definition \ref{def:transition} gives
	\begin{equation}\label{eq:lemmacase2}
		\sum_{ \fluentState{S}' \in \fluentState{\mathcal{S}}} t_\mathcal{D} (S,\fluentState{S}') = \sum_{\mathclap{ \fluentState{S}' \in \fluentState{\mathcal{S}} }}  \pi ( \tset_\mathcal{D} (S, \fluentState{S}') )
	\end{equation}
	
	Notice that for a fixed outcome $O$, it is impossible to have $O\in\tset_\mathcal{D} (S,\fluentState{S}')$ and $O\in\tset_\mathcal{D} (S,\fluentState{S}'')$ for two distinct fluent states $\fluentState{S}', \fluentState{S}''$ as this would imply $\fluentState{S}'=(S\restriction\mathcal{F})\oplus\chi(O)=\fluentState{S}''$. Hence it is sufficient to show that $\{ O\in\tset_\mathcal{D} (S,\fluentState{S}') \mid \fluentState{S}' \in \fluentState{\mathcal{S}} \} = \head(c) $, as this implies that the sum \eqref{eq:lemmacase2} equals $1$ since $\weight( \head(c) )=1$ by definition of a c-proposition.
	
	By definition of a transition set, $\{ O\in\tset_\mathcal{D} (S,\fluentState{S}') \mid \fluentState{S}' \in \fluentState{\mathcal{S}} \} \subseteq \head(c)$. Conversely, for any $O\in \head(c)$, $O \in \tset_\mathcal{D} (S, \fluentState{S}')$ for $\fluentState{S}' = (S\restriction\mathcal{F}) \oplus \chi(O)$, hence $ \head(c) \subseteq \{ O\in\tset_\mathcal{D} (S,\fluentState{S}') \mid \fluentState{S}' \in \fluentState{\mathcal{S}} \} $ which ends the proof of lemma.
\end{proof}

\begin{lemma}\label{lem:e}
	Let  $ \mathcal{D} $ be any domain description, $ I $ be any instant and $ N_I $ be the possibly empty action narrative that contains exactly those p-propositions in $ \mathcal{D} $ that occur at $ I $. Let $ W_1, \dots, W_m $ be well-behaved worlds w.r.t. $ \mathcal{D} $ such that $ W_i(I)\restriction\mathcal{F} = W_j(I)\restriction\mathcal{F} $ for all $1\leq i,j \leq m$ and which represent all the equivalence classes such that $ W $ is equivalent to $ W' $ if and only if $ W(I)=W'(I) $. Then, \[ \sum_{j=1}^m \epsilon(N_I, W_j) = 1 \]
\end{lemma}
\begin{proof}
	Let $p_1, \dots, p_k$ be the p-propositions in $ N_I $ (possibly none, in which case $k=0$) that have probabilities strictly less than $ 1 $ attached. Then, there are at least $2^k$ well-behaved worlds w.r.t. $ \mathcal{D} $ satisfying $ W(I) \restriction \mathcal{F} =W'(I) \restriction \mathcal{F} $, of which $W_1, \dots, W_m$ are representatives of each possible assignment of actions to $ \{\True,\False\} $. Therefore, if we let $ p_i$ have the form ``$ A_i \performedAt I \withProb P_i $'' for all $ 1\leq i \leq k $, the sum $ \sum_{j=1}^m \epsilon(N_I, W_j) $ evaluates to:
\[ P_1 \cdot P_2 \dots P_k + P_1 \cdot P_2 \dots (1-P_k) + \dots + (1-P_1) \cdot (1-P_2) \dots (1-P_k) = 1 \]
\end{proof}

We can now prove the central property of $ M^* $:

\begin{proposition} Given a model $\model{\mathcal{D}}$ of a domain description $\mathcal{D}$, its extension to $\modelStar{\mathcal{D}}$ is a probability function.
	\begin{proof} We show that for any domain description, requirements \ref{item:probabilityFunction1} and \ref{item:probabilityFunction2} as in Definition \ref{def:probabilityFunction} are always satisfied by a model of that domain description.
    
    \textbf{Proof of requirement \ref{item:probabilityFunction1}.} We need to show that for any $\psi$ such that $W\mmodels\psi$ for all worlds $W$, 
		\begin{equation} \label{eq:sumallworlds}
			\modelStar{\mathcal{D}}{\psi}=\sum_{\mathclap{W\in\mathcal{W}}} \model{\mathcal{D}}{W} = \sum_{\mathclap{W\in\mathcal{W}_\mathcal{D} }} \model{\mathcal{D}}{W} = 1.
		\end{equation}
		where the second equality is guaranteed by the fact that $\model{\mathcal{D}}{W} = 0$ when $W$ is not well-behaved.
				
		For a p-proposition $p$ we define $\hasInstant(p)$ as the instant $p$ has, and for an action narrative $N$ we extend this to:
		\[
			\hasInstants (N)= \bigcup_{\mathclap{ p \in N }} \hasInstant(p)
		\]
		
		We prove \eqref{eq:sumallworlds} by induction on $\hasInstants( \narr(D) ) = \{ I_1, \dots, I_n \}$ where $I_1, \dots, I_n$ are ordered w.r.t. $\leq$. Notice that $\mathcal{D}_{\leq I_n} = \mathcal{D}$.
		
		\textit{Base case.} We consider $\mathcal{D}_{\emptyset}$ first. Since there are no p-propositions in $\mathcal{D}_{\emptyset}$, $\hasInstants( \mathcal{D}_{\emptyset} )=\emptyset$, $\epsilon_{\mathcal{D}_\emptyset}(W) = \emptyset$ for all worlds $ W $, and the sum \eqref{eq:sumallworlds} becomes:
		\begin{equation}\label{eq:basecase}
			\modelStar{\mathcal{D}_{\emptyset}}{\psi}=  \sum_{\mathclap{ {W\in\mathcal{W}_{\mathcal{D}_{\emptyset}} }}} ~ \left( \sum_{ \tr \in TR^W_{\mathcal{D}_{\emptyset}} } \mkern-18mu \pi( \tr(\initialChoiceSymbol) ) \right)
		\end{equation}

		Let $ \{ O_1, \dots, O_m \} $ be the outcomes occurring in the only i-proposition of $\mathcal{D}_{\emptyset}$. We prove that the well-behaved worlds w.r.t. $\mathcal{D}_{\emptyset}$ are exactly those $W$s taking the form $W( I ) \restriction \mathcal{F} = \chi(O_i)$ and $\neg A \in W(I)$ for all instants $I$ and all action symbols $A$.
		
		If $W$ has this form, then it satisfies CWA (as there are no p-propositions in $\mathcal{D}_{\emptyset}$, and this is consistent with $\neg A \in W(I)$ for all $I$ and $A$), it satisfies the initial condition w.r.t. $\mathcal{D}_{\emptyset}$ as $O_i$ is an initial choice w.r.t. $\mathcal{D}_{\emptyset}$ and $W(\initialInstant)\restriction\mathcal{F} = \chi(O_i)$ by definition, and finally it also satisfies the justified change condition in the form \eqref{eq:justifiedChange} as $\occ_{\mathcal{D}_{\emptyset}} (W) = \emptyset$, which in turn forces $W(I)\restriction\mathcal{F} = W(I')\restriction\mathcal{F}$ for all $I$ and $I'$. The fact that if $W$ is well-behaved then it is of the form above is a simple inversion of the previous chain of implications.
		
		Notice that each of these well-behaved worlds is consistent with a unique trace $\langle O_i@\initialChoiceSymbol \rangle$ for some $1\leq i\leq m$, and let $W_i$ denote the world having trace $\langle O_i@\initialChoiceSymbol \rangle$ for $1\leq i\leq m$. Hence we can write $\mathcal{W}_{\mathcal{D}_{\emptyset}} = \{ W_1, \dots, W_m \}$. For such $W_i$, \[ \sum_{\mathclap{{\tr \in TR^{W_i}_{\mathcal{D}_{\emptyset}}}}} \pi( \tr(\initialChoiceSymbol) ) = \pi(O_i)  \] and \eqref{eq:basecase} evaluates to:
		\begin{equation}
			\modelStar{\mathcal{D}_{\emptyset}}{\psi}= \sum_{\mathclap{W_i \in \mathcal{W}_{\mathcal{D}_{\emptyset}} }} \pi(O_i) = \sum_{ i=1 }^{m} \pi(O_i)= 1
		\end{equation}
		as $\weight( \{ O_1, \dots, O_m \} ) = 1$ by Definition \ref{def:iProposition} of an i-proposition.
		
		\textit{Inductive step.} Assume that $\modelStar{\mathcal{D}_{<I_i}}{\psi} = 1$ for some $i \leq n$. We prove that $\modelStar{\mathcal{D}_{\leq I_i}}{\psi} = 1$.
		
		
		Let $[W']^I_\mathcal{D}$ be the set of well-behaved worlds w.r.t. $ \mathcal{D} $ that are indistinguishable from $ W' $ up to $ I $. Corollary \ref{cor:decomposingw} and Lemma \ref{lem:t} together with the inductive hypothesis allow us to turn Equation \eqref{eq:sumallworlds} into:
		\begin{align*}
			\modelStar{\mathcal{D}_{\leq I_{i}}}{\psi} &= \sum_{\mathclap{W \in \mathcal{W}_{\mathcal{D}_{\leq I_i} } }} \model{\mathcal{D}_{\leq I_{i}}}{W} \\
				&\stackrel{\text{Cor.\ref{cor:decomposingw}}}{=} \sum_{\mathclap{W' \in \mathcal{W}_{\mathcal{D}_{< I_i} } }} ~~~~~ \sum_{W \in [W']^{I_i}_{\mathcal{D}_{\leq I_i}}} \mkern-18mu \frac{\epsilon_{\mathcal{D}_{\leq I_i}} ( W )}{\epsilon_{\mathcal{D}_{<I_i}} ( W' )} \cdot \model{\mathcal{D}_{<I_i}}{W'} \cdot t_\mathcal{D} ( W(I_i), \fluentState{S}^{W}_{>I_i} ) \\
				&= \sum_{\mathclap{W' \in \mathcal{W}_{\mathcal{D}_{< I_i} } }} \model{\mathcal{D}_{<I_i}}{W'} \mkern-18mu \sum_{{W \in [W']^{I_i}_{\mathcal{D}_{\leq I_i}}}} \mkern-18mu \frac{\epsilon_{\mathcal{D}_{\leq I_i}} ( W )}{\epsilon_{\mathcal{D}_{<I_i}} ( W' )} \cdot t_\mathcal{D} ( W(I_i), \fluentState{S}^{W}_{>I_i} )
		\end{align*}
		
	According to Corollary \ref{cor:decomposingw} every world in $ [W']^{I_i}_{\mathcal{D}_{\leq I_i}} $ can be reconstructed from its state at instant $I_i$ and the unique state $ \fluentState{S}' $ that it takes at instants strictly greater than $I_i$. Consider the equivalence relation such that two well-behaved worlds $W$ and $W'$ w.r.t. $ \mathcal{D}_{\leq I_i} $ are equivalent if and only if $ W(I_i)=W'(I_i) $, and let $ W_1, \dots, W_m $ be representatives of all the equivalence classes. Then, the above chain of equalities continues as follows:
	
		\begin{align*}
			&= \sum_{\mathclap{W' \in \mathcal{W}_{\mathcal{D}_{< I_i} } }} \model{\mathcal{D}_{<I_i}}{W'} \sum_{j=1}^{m} \frac{\epsilon_{\mathcal{D}_{\leq I_i}} ( W_j )}{\epsilon_{\mathcal{D}_{<I_i}} ( W' )} \cdot \sum_{ \fluentState{S}' \in \fluentState{\mathcal{S}} } t_\mathcal{D} ( W_j(I_i), \fluentState{S}' ) \\
			&\stackrel{\text{Lem.\ref{lem:t}}}{=} \sum_{\mathclap{W' \in \mathcal{W}_{\mathcal{D}_{< I_i} } }} \model{\mathcal{D}_{<I_i}}{W'} \sum_{j=1}^{m} \frac{\epsilon_{\mathcal{D}_{\leq I_i}} ( W_j )}{\epsilon_{\mathcal{D}_{<I_i}} ( W' )}\\
			&\stackrel{\text{Lem.\ref{lem:e}}}{=} \sum_{\mathclap{W' \in \mathcal{W}_{\mathcal{D}_{< I_i} } }} \model{\mathcal{D}_{<I_i}}{W'} \stackrel{\text{Ind.Hyp.}}{=} 1
		\end{align*}
		
	\textbf{Proof of requirement \ref{item:probabilityFunction2}.} Let $\varphi$ and $\psi$ be two i-formulas such that $\varphi \mmodels \neg \psi$. Obviously, since $\varphi \mmodels \neg \psi$ if for some $W\in \mathcal{W}$, $W \mmodels \varphi$, then $W\notmmodels \psi$ and vice-versa, hence \[ M^*_\mathcal{D} (\varphi \vee \psi) = \sum_{\mathclap{W \mmodels \varphi \vee \psi}} M_\mathcal{D} (W) = \sum_{\mathclap{W \mmodels \varphi}} M_\mathcal{D} (W) + \sum_{\mathclap{W \mmodels \psi}} M_\mathcal{D} (W) = M^*_\mathcal{D}(\varphi) + M^*_\mathcal{D} (\psi). \]

	\end{proof}
\end{proposition}

An immediate consequence of the previous proposition is the following one:

\begin{corollary} For any given domain description $ \mathcal{D} $, $\mathcal{W}_\mathcal{D} \neq \emptyset$.
\end{corollary}

\subsection{Example entailments}\label{sec:entailments}
The following are example entailments from the formalisation of Scenarios \ref{scenario:coin}, \ref{scenario:antibiotic} and \ref{scenario:keys}.

$ \mathcal{D}_C $ as in Example \ref{ex:coindomain} entails, among others:\\

\noindent
\hspace*{\axiomsIndent}$\top \holdsWithProb 1$\hfill($\mmodels$C1)\\

\noindent
\hspace*{\axiomsIndent}$\ifor{\Coin=\Tails}{0} \holdsWithProb 0$\hfill($\mmodels$C2)\\

\noindent
\hspace*{\axiomsIndent}$\ifor{\Toss=\True}{1} \holdsWithProb 1$\hfill($\mmodels$C3)\\

\noindent
\hspace*{\axiomsIndent}$\ifor{\Coin=\Heads}{2} \holdsWithProb 0.51$\hfill($\mmodels$C4)\\

\noindent
\hspace*{\axiomsIndent}$\ifor{\Coin=\Heads}{1} \wedge \ifor{\Coin=\Tails}{3}$\hfill($\mmodels$C5)\\
\hspace*{1.5em}$\holdsWithProb 0.49$\\

\noindent
where $\top$ is any tautological i-formula (i.e., $W \mmodels \top$ for all $W\in\mathcal{W}$).

The following h-propositions are entailed by $ \mathcal{D}_A $:\\

\noindent
\hspace*{\axiomsIndent}$\ifor{\Bacteria=\Weak}{0} \holdsWithProb 0.9$\hfill($\mmodels$A1)\\

\noindent
\hspace*{\axiomsIndent}$\ifor{\Bacteria=\Weak \wedge \Rash=\Absent}{0}$\hfill($\mmodels$A2)\\
\hspace*{\axiomsMedIndent}$\holdsWithProb 0$\\

\noindent
\hspace*{\axiomsIndent}$\ifor{\Bacteria=\Resistant}{2} \holdsWithProb 0.27$\hfill($\mmodels$A3)\\

\noindent
\hspace*{\axiomsIndent}$\ifor{\Rash=\Absent}{4} \holdsWithProb 0.733846$\hfill($\mmodels$A4)\\

\noindent
\hspace*{\axiomsIndent}$\ifor{\Bacteria=\Absent \wedge \Rash=\Absent }{4}$\hfill($\mmodels$A5)\\
\hspace*{\axiomsMedIndent}$\holdsWithProb 0.650769$\\

Notice that from ($\mmodels$A4) and ($ \mmodels $A5) we can calculate the conditional probability that the medicine has cured the infection at instant $4$, i.e. $\ifor{\Bacteria=\Absent}{4}$, given that no sign of rash is visible at the end of the treatment, i.e. $\ifor{\Rash=\Absent}{4}$. Applying \eqref{eq:conditionalprobability} gives that this probability equals $0.650769/0.733846 \approx 0.887$.

Finally, the following h-propositions are entailed by $ \mathcal{D}_K $:\\

	\noindent
	\hspace*{\axiomsIndent}$\ifor{\LockedOut}{\text{8AM}} \holdsWithProb 0.01$,\hfill($\mmodels$K1)\\

	\noindent	
	\hspace*{\axiomsIndent}$\ifor{\HasKeys}{\text{9AM}} \holdsWithProb 0.99$,\hfill($\mmodels$K2)\\
	
	\noindent	
	\hspace*{\axiomsIndent}$\ifor{\PickupKeys}{\text{7:40AM}} \holdsWithProb 0.99$.\hfill($\mmodels$K3)
\section{Translation}

To aid the reader's intuition, we outline the translation of a domain description $\mathcal{D}$ into an answer set program. The idea is that of generating all the traces of a domain description as distinct stable models of the translated domain description. These traces can then be processed by an external tool such as AWK in order to calculate the probability of given queries.

In the following, we restrict the domain language to be such that $\mathcal{I}$ is a finite interval $\{0,1,\dots,\maxinst\}$ of $\mathbb{N}$, with $\initialInstant=0$ and $\leq = \leq_\mathbb{N}$ being the usual ordering relation between naturals.

\subsection{Translation of the domain-dependent part}\label{sec:translationdomaindependent}
We start by introducing the full translation of the coin domain description from Example \ref{ex:coindomain}.

\begin{example}[Translation of the Coin Toss Domain]\label{ex:translationcoin} Let $\mathcal{D}_C$ be as in Example \ref{ex:coindomain}. The translation of $\mathcal{D}_C$ results in the following set of clauses:\\

\noindent
\hspace*{\axiomsIndent}$\fluent(\coin).$ \hfill(TC0)\\
\hspace*{\axiomsIndent}$\action(\toss).$\\
\hspace*{\axiomsIndent}$\instant(0..\maxinst).$\\
	
\noindent
\hspace*{\axiomsIndent}$\possVal( \coin, \heads ).$\hfill(TC1)\\
\hspace*{\axiomsIndent}$\possVal( \coin, \tails ).$\hfill\\

\noindent
\hspace*{\axiomsIndent}$\belongsTo( (\coin,\heads), id^0_1 ).$\hfill(TC2)\\
\hspace*{\axiomsIndent}$\initialCondition ( (id^0_1, 1) ).$\hfill\\

\noindent
\hspace*{\axiomsIndent}$\belongsTo( (\coin,\heads), id^1_1 ).$\hfill(TC3.1)\\
\hspace*{\axiomsIndent}$\causesOutcome ( (id^1_1, 49/100 ), I ) \leftarrow $\hfill\\
\hspace*{\axiomsMedIndent}$ \holds( ((\toss,\true), I) ). $\hfill\\

\noindent
\hspace*{\axiomsIndent}$\belongsTo( (\coin,\tails), id^1_2 ).$\hfill(TC3.2)\\
\hspace*{\axiomsIndent}$\causesOutcome ( (id^1_2, 49/100 ), I ) \leftarrow $\hfill\\
\hspace*{\axiomsMedIndent}$ \holds( ((\toss,\true), I) ). $\hfill\\

\noindent
\hspace*{\axiomsIndent}$\causesOutcome ( (id^1_3, 2/100 ), I ) \leftarrow $\hfill(TC3.3)\\
\hspace*{\axiomsMedIndent}$ \holds( ((\toss,\true), I) ). $\hfill\\

\noindent
\hspace*{\axiomsIndent}$\performed( \toss, 1 ).$\hfill(TC4)\\

\noindent where, informally, the set of clauses (TC0) is the translation of the three sorts $\mathcal{F}$, $ \mathcal{A} $ and $ \mathcal{I} $; (TC1), (TC2) and (TC4) are the translation of (C1), (C2) and (C4) respectively; (TC3.1), (TC3.2) and (TC3.3) together give the translation of (C3), and each of them corresponds to an outcome in the corresponding c-proposition;
\end{example}

Since in logic programming lowercase letters are conventionally used for constants, we switch to that convention by letting lower case letters be the logic programming counterparts of (upper case) constants in PEC so that e.g. $f$ is regarded as the translated fluent $F$. Furthermore, literals of the form $X=V$ are translated into pairs of the form $(x,v)$.

The three sorts $\mathcal{F}$, $ \mathcal{A} $ and $\mathcal{I}$ are translated to the three sets $\{ \fluent(f) \mid F\in\mathcal{F} \}$, $\{ \action(a) \mid A\in\mathcal{A} \}$ and $\{ \instant(i) \mid I\in\mathcal{I} \}$ respectively (see e.g. (TC0) in Example \ref{ex:translationcoin}).

Let $c$ be a c-proposition of the generic form \eqref{eq:cProposition}:
\[ \theta \causesOneOf \{ O_1, O_2, \dots, O_m \} \]
but first considering the case where $\theta$ is a conjunction of the form $X_1=V_1\wedge\dots\wedge X_j=V_j$. Given a conjunction $ \theta$ as before, we write $\holds( \ifor{\theta}{I} ) $ as a shorthand for the logic programming conjunction \[ \holds( ((x_1,v_1),I) ), \dots, \holds( ((x_j,v_j),I) ). \]

Fix an enumeration (without repetitions) of all the c-propositions in $\mathcal{D}$, and let $c$ be the $n$th proposition occurring in such enumeration. Then, $c$ is translated to:\\

\noindent
\hspace*{\axiomsIndent}$\{ \, \belongsTo( (x,v), id^n_i ) \mid i=1, \dots, m, X=V \in\chi(O_i) \, \}$\\
\hspace*{\axiomsMedIndent}$\cup \{ \, \causesOutcome( (id^n_i, p), I ) \leftarrow \holds(\theta, I) \mid$\\
\hspace*{\axiomsBigIndent}$i=1,\dots,m, \, , P=\pi(O_i) \, \} $\\

\noindent where $id^n_1, \dots, id^n_m$ are new constants in the underlying ASP language. We write $ C_\mathcal{D} $ for the set of all translated c-propositions in $ \mathcal{D} $.

\begin{example}\label{ex:translationcproposition} The clauses (TC3.1), (TC3.2) and (TC3.3) are the translation of the c-proposition (C2) from Example \ref{ex:coindomain}.
	
	As a further example, consider the c-proposition (A5) as in Example \ref{ex:antibioticdomain}, and notice that two outcomes occur in it, i.e. $ (\{\Bacteria=\Absent, \Rash=\Absent\}, 1/13) $ and $ ( \emptyset , 4/13) $. If we fix the enumeration of c-propositions in $\mathcal{D}_A$ such that (A4) is first and (A5) is second, this c-proposition is translated to:\\
	
	\noindent
	\hspace*{\axiomsIndent}$ \belongsTo( (\bacteria,\absent), id^2_1 ).$\hfill(TA5.1)\\
	\hspace*{\axiomsIndent}$ \belongsTo( (\rash,\absent), id^2_1 ).$\\
	\hspace*{\axiomsIndent}$ \causesOutcome( (id^2_1, 1/13),I ) \leftarrow $\\
	\hspace*{\axiomsMedIndent}$ \holds( (\takesMedicine,\true), I ),$\\
	\hspace*{\axiomsMedIndent}$\holds( (\bacteria,\resistant), I ). $\\
	
	\noindent
	\hspace*{\axiomsIndent}$ \causesOutcome( (id^2_2, 4/13),I ) \leftarrow $\hfill(TA5.2)\\
	\hspace*{\axiomsMedIndent}$ \holds( (\takesMedicine,\true), I ),$\\
	\hspace*{\axiomsMedIndent}$\holds( (\bacteria,\resistant), I ). $\\
\end{example}


If $\theta$ is not a conjunction of literals, then represent it in Disjunctive Normal Form, i.e. in the form $\theta_1 \vee \dots \vee \theta_n$ with $\theta_1, \dots, \theta_n$ conjunctions of literals, and then for each rule write the precondition of each\causesOneOf clause in the disjunctive form:
\[
	\holds(\ifor{\theta_1}{I}); \, \dots ; \, \holds(\ifor{\theta_n}{I}).
\]

The translation of i-propositions works in a very similar way: if $J$ is an i-proposition of the general form \eqref{eq:iProposition}:
\[ \initiallyOneOf \{ O_1, O_2, \dots, O_m \} \]
then its translation is given by the following set of clauses:\\

\noindent
\hspace*{\axiomsIndent}$\{ \, \belongsTo( (x,v), id^0_i ) \mid i=1, \dots, m, X=V\in\chi(O_i) \, \}$\\
\hspace*{\axiomsMedIndent}$\cup \{ \, \initialCondition( (id^0_i,p) ) \mid i=1,\dots,m, P=\pi(O_i) \, \} $\\

\noindent and we write $ I_\mathcal{D} $ for the set of all translated i-propositions in $ \mathcal{D} $.

\begin{example}
	An example of translated i-proposition is the set of clauses (TC2), that translate (C2) as in Example \ref{ex:translationcoin}.
	
	The i-proposition (A3) from Example \ref{ex:antibioticdomain} is translated to:\\
	
	\noindent
	\hspace*{\axiomsIndent}$ \belongsTo( (\bacteria,\weak), id^0_1 ).$\hfill(TA3.1)\\
	\hspace*{\axiomsIndent}$ \belongsTo( (\rash,\present), id^0_1 ).$\\
	\hspace*{\axiomsIndent}$ \initialCondition( (id^0_1, 9/10) ). $\\
	
	\noindent
	\hspace*{\axiomsIndent}$ \belongsTo( (\bacteria,\absent), id^0_2 ).$\hfill(TA3.2)\\
	\hspace*{\axiomsIndent}$ \belongsTo( (\rash,\present), id^0_2 ).$\\
	\hspace*{\axiomsIndent}$ \initialCondition( (id^0_2, 1/10) ). $\\
\end{example}

Finally, the translation of p-propositions and v-propositions is straightforward: any generic p-proposition of the form \eqref{eq:pProposition} is translated to\\

\noindent
\hspace*{\axiomsIndent}$\performed(a,i).$\\

\noindent and we write $ P_\mathcal{D} $ for the set of all translated p-propositions in $ \mathcal{D} $, while any v-proposition of the form \eqref{eq:vProposition} is translated to:\\

\noindent
\hspace*{\axiomsIndent}$\{ \, \possVal( f, v_i ) \mid 1\leq i \leq n \, \}$\\

\noindent and we write $ V_\mathcal{D} $ for the set of all translated v-propositions in $ \mathcal{D} $.

\begin{example}
	The v-proposition (C1) and p-proposition (C4) from from Example \ref{ex:coindomain} are translated to (TC1) and (TC4) as in Example \ref{ex:translationcoin}, respectively, while (A1), (A2), (A6), (A7) from Example \ref{ex:antibioticdomain} are translated to:\\
	
	\noindent
	\hspace*{\axiomsIndent}$ \possVal( \bacteria, \weak ). $\hfill(TA1)\\
	\hspace*{\axiomsIndent}$ \possVal( \bacteria, \resistant ). $\\
	\hspace*{\axiomsIndent}$ \possVal( \bacteria, \absent ). $\\
	
	\noindent
	\hspace*{\axiomsIndent}$ \possVal( \rash, \present ). $\hfill(TA2)\\
	\hspace*{\axiomsIndent}$ \possVal( \rash, \absent ). $\\
	
	\noindent
	\hspace*{\axiomsIndent}$\performed(\takesMedicine,1).$\hfill(TA6)\\
	
	\noindent
	\hspace*{\axiomsIndent}$\performed(\takesMedicine,3).$\hfill(TA7)
\end{example}

Since $P_\mathcal{D}$ and $V_\mathcal{D}$ contain only ground facts which clearly correspond to their semantic counterparts (i.e., p-propositions and v-propositions) we are not going to discuss their correctness in close detail.

We write $\Pi_\mathcal{D}$ for the set of translated propositions from $ \mathcal{D} $, e.g. if $ \mathcal{D}_C $ is the coin toss domain, $ \Pi_{\mathcal{D}_C} =$ (TC0--4).

\subsection{Translation of the domain-independent part}
We define the domain-independent part of our theory to be:\\

\noindent
\hspace*{\axiomsIndent}$\possVal(A,\true) \leftarrow \action(A). $\hfill(PEC1)\\
\hspace*{\axiomsIndent}$\possVal(A,\false) \leftarrow \action(A).$\\

\noindent
\hspace*{\axiomsIndent}$\fluentOrAction(X) \leftarrow \fluent(X); \, \action(X). $\hfill(PEC2)\\

\noindent
\hspace*{\axiomsIndent}$\literal( (X,V) ) \leftarrow \possVal(X,V). $\hfill(PEC3)\\

\noindent
\hspace*{\axiomsIndent}$\iLiteral( (L,I) ) \leftarrow \literal(L), \, \instant(I). $\hfill(PEC4)\\

\noindent
\hspace*{\axiomsIndent}$\definitelyPerformed(A,I) \leftarrow \performed(A,I,1).$\hfill (PEC5)\\

\noindent
\hspace*{\axiomsIndent}$\possiblyPerformed(A,I) \leftarrow \performed(A,I,P).$\hfill (PEC6)\\

\noindent
\hspace*{\axiomsIndent}$1\{ \  \holds(((X,V),I)) : \iLiteral(((X,V),I))  \  \}1 $\hfill(PEC7)\\
\hspace*{\axiomsMedIndent}$\leftarrow \instant(I), \, \fluentOrAction(X).$\\

\noindent
\hspace*{\axiomsIndent}$\inOcc(I) \leftarrow \instant(I), \, \causesOutcome( O, I ).$ \hfill (PEC8)\\

\noindent
\hspace*{\axiomsIndent}$1\{ \ \effectChoice( O, I ) : \causesOutcome( O, I ) \ \}1$ \hfill(PEC9) \\
\hspace*{\axiomsMedIndent}$\leftarrow \inOcc(I).$\\

\noindent
\hspace*{\axiomsIndent}$1\{ \  \initialChoice( O ) : \initialCondition( O ) \  \}1.$ \hfill(PEC10)\\

\noindent
\hspace*{\axiomsIndent}$\bot \leftarrow \action(A), \, \instant(I),$\hfill (PEC11)\\
\hspace*{\axiomsMedIndent}$\holds( ((A,\true),I) ), \, not \, \possiblyPerformed(A,I).$\\

\noindent
\hspace*{\axiomsIndent}$\bot \leftarrow \action(A), \, \instant(I),$\hfill (PEC12)\\
\hspace*{\axiomsMedIndent}$\holds( ((A,\false),I) ), \, \definitelyPerformed(A,I).$\\

\noindent
\hspace*{\axiomsIndent}$\bot \leftarrow \initialChoice( (S,P) ), \, \literal(L), $\hfill (PEC13)\\
\hspace*{\axiomsMedIndent}$\belongsTo( L, S ), \, not \, \holds( (L,0) ).$\\

\noindent
\hspace*{\axiomsIndent}$\bot \leftarrow \instant(I), \, \effectChoice( (X,P) , I), \, $\hfill (PEC14)\\
\hspace*{\axiomsMedIndent}$ \fluent(F), \, \belongsTo((F,V), X), $\\
\hspace*{\axiomsMedIndent}$not \, \holds( ((F,V),I+1) ), \, I<\maxinst.$\\

\noindent
\hspace*{\axiomsIndent}$\bot \leftarrow \instant(I), \, \fluent(F), \, not \, \holds( ((F,V),I) ), $\hfill (PEC15)\\
\hspace*{\axiomsMedIndent}$\effectChoice( (X,P) , I), \, not \, \belongsTo((F,V), X), $\\
\hspace*{\axiomsMedIndent}$\holds( ((F,V),I+1) ), \, I<\maxinst.$\\

\noindent
\hspace*{\axiomsIndent}$\bot \leftarrow \fluent(F), \, \instant(I), \, \holds( ((F,V),I) ), $\hfill (PEC16)\\
\hspace*{\axiomsMedIndent}$not \, \inOcc(I), \, not \, \holds( ((F,V),I+1) ),$\\
\hspace*{\axiomsMedIndent}$I<\maxinst.$\\

\noindent
\hspace*{\axiomsIndent}$\eval(A,I,P) \leftarrow \action(A), \, \instant(I), $\hfill (PEC17)\\
\hspace*{\axiomsMedIndent}$\performed(A,I,P), \, \holds( ((A,\true),I) ).$\\

\noindent
\hspace*{\axiomsIndent}$\eval(A,I,1-P) \leftarrow \action(A), \, \instant(I), $\hfill (PEC18)\\
\hspace*{\axiomsMedIndent}$\performed(A,I,P), \, \holds( ((A,\false),I) ).$\\

(PEC1--4) implement the basic predicates and sorts of PEC, namely: (PEC1) states that actions are boolean; (PEC2) defines a characteristic predicate for $\mathcal{F} \cup \mathcal{A}$; (PEC3) and (PEC4) define literals and i-literals, respectively. (PEC5) and (PEC6) define the two auxiliary predicates $\definitelyPerformed$ and $\possiblyPerformed$ representing the sets of actions and instants such that $ A $ is certainly performed at $ I $ (i.e., with probability 1) and such that $ A $ might have been performed at $ I $ (i.e., with a probability greater than $ 0 $) respectively. Proving that (PEC1--6) correctly characterise the sorts and sets they stand for is trivial and is omitted here.

Intuitively, axioms (PEC7--18) correspond to the definitions introduced in the previous section, namely: (PEC7) corresponds to Definition \ref{def:worlds}, (PEC8) defines a characteristic predicate for $\occ$ as in Definition \ref{def:causeOccurrence}, (PEC9) and (PEC14--16) correspond to justified change, (PEC10) and (PEC13) corresponds to the initial condition, (PEC11) and (PEC12) correspond to CWA for actions.

We denote the domain-independent part of our theory, i.e. (PEC1--18), by $\Pi_I$. Notice that axioms (PEC11--16) are constraints, and in the following will be referred to as $ \Pi_C $.
\section{Correctness}
We now show that the provided translation is sound and complete with respect to the definitions given in the previous sections. This proof relies on the Splitting Theorem \cite{lifschitz1994splitting}, a useful tool to obtain the answer sets of a ground program. Informally, a set $U$ of atoms is a splitting set for a program $\Pi$ if, for every rule in $\Pi$, if $U$ contains some atom in the head of such rule, then it also contains all the atoms occurring in that rule. For instance, if $\Pi' = \{ a \leftarrow not \, b, b\leftarrow c, c \}$ then $\{a,b,c\}$, $\emptyset$, $\{ b,c \}$ and $\{ c \}$ are splitting sets for $\Pi'$, whereas $\{ a, b \}$, $\{ a \}$ and $\{ b \}$ are not.

A splitting set $U$ splits an answer set program $\Pi$ into a bottom program $\botProgram{U}{\Pi}$ and a top program $\topProgram{U}{\Pi} = \Pi \setminus \botProgram{U}{\Pi}$. With the program $\Pi'$ defined as above, $U'=\{ c \}$ splits $\Pi'$ into $\botProgram{U'}{\Pi'} = \{ c \}$ and $\topProgram{U'}{\Pi'} = \{ a \leftarrow not \,  b, b \leftarrow c \}$.

The splitting set theorem states that the answer sets of $\Pi$ are exactly those that can be expressed as $X \cup Y$ for $X$ an answer set of $\botProgram{U}{\Pi}$ and $Y$ an answer set of $\evaluation{U}{\topProgram{U}{\Pi},X}$, where $\evaluation{U}{\Pi, Z}$ for a generic program $\Pi$, set of atoms $U$ and answer set $Z$ denotes the \emph{partial evaluation of the program $ \Pi $ w.r.t. $ U $} which is defined as follows: a rule $r$ is in $\evaluation{U}{\Pi,Z}$ if and only if there exists a rule $r' \in \Pi$ such that all literals in the body of $r'$ with at least an atom of $U$ occurring in them are also in $Z$, and the rule $r$ is obtained from $r'$ by removing all the occurrences of such literals. If we consider $\Pi'$ and $U'$ again and let $X'$ be the only answer set $\{ c \}$ of $\botProgram{U'}{\Pi'} = \{ c \}$, $\Pi'' = \evaluation{U'}{\topProgram{U'}{\Pi'},X'} = \{ a \leftarrow not \,  b, b \}$ and notice that now we can split $\Pi''$ itself. If we let $U''=\{ b \}$, then $ \botProgram{U''}{\Pi''} = \{ b \} $ and $\Pi''' = \evaluation{U''}{\topProgram{U''}{\Pi''},X''} = \emptyset$ for the only answer set $X''=\{ b \}$ of $\botProgram{U''}{\Pi''}$. The answer sets of the original program $\Pi'$ can now be obtained as $X' \cup X'' \cup X'''$, where $X'=\{ c \}$ is the answer set of $\botProgram{U'}{\Pi'}$, $X''=\{ b \}$ is the answer set of $\botProgram{U''}{\Pi''} = \{ c \}$ and $X''' = \emptyset$ is the answer set of $\Pi'''$. Then, the program $\Pi'$ has only one answer set $\{ b, c \}$.

In the following, we will use the fact that answer sets of a choice rule $\{ a_1, \dots, a_n \}$ are the power set $\{ \emptyset, \{a_1\}, \dots, \{a_n\}, \{a_1,a_2\}, \dots, \{ a_1, \dots, a_n \} \}$, and that answer sets of a constrained choice rule $X \{ a_1, \dots, a_n \} Y$ are the answer sets of $\{ a_1, \dots, a_n \}$ with cardinality $\geq X$ and $\leq Y$. Also, we use the fact that the only answer set of the program $\{ {p(X):q(X)}, q(a_1), \dots, q(a_n) \}$, where ${p(X):q(X)}$ is called a conditional literal, is $\{ p(a_1), \dots, p(a_n), q(a_1), \dots, q(a_n) \}$. Notice that conditional literal and choice rules can be combined so that e.g. answer sets of the program $\{ q(a,b), q(a,c), q(b,c), 1\{p(X):q(a,X)\}1 \}$ are $\{ q(a,b), q(a,c), q(b,c), p(b) \}$ and $\{ q(a,b), q(a,c), q(b,c), p(c) \}$. Finally, constraints are used to eliminate answer sets that satisfy its body, e.g. answer sets of the program $\{ q(a,b), q(a,c), q(b,c), 1\{p(X):q(a,X)\}1, {\bot \leftarrow p(b)} \}$ are the answer sets of $\{ q(a,b), q(a,c), q(b,c), 1\{p(X):q(a,X)\}1 \}$ that do not satisfy $p(b)$, hence $\{ q(a,b), q(a,c), q(b,c), p(c) \}$ is its only answer set.

The splitting set theorem can only be applied to ground programs, hence in the following we will interpret non-ground clauses as shorthand for the set of all their ground instances, e.g. the clause $p(X) \leftarrow q(X,Y)$ from the program $\{ p(X) \leftarrow q(X,Y), q(a,b) \}$ is shorthand for the set $\{ p(a) \leftarrow q(a,a), p(a) \leftarrow q(a,b), p(b) \leftarrow q(b,a), p(b) \leftarrow q(b,b) \}$.

Before proving the correctness of the implementation, we need to define a correspondence between answer sets and traces. This is the aim of the following definitions:

\begin{definition}[Manifest Choice Element] We say that the choice element $ \choiceElement{\fluentState{X}}{P^+}{I} $ is \emph{manifest in the answer set $ Z $} if and only if there exists a symbol $ \id $ such that $ \effectChoice( (\id,p^+) , i ) \in Z $ and such that $L \in \fluentState{X}$ if and only if $ \belongsTo(l, \id) \in Z$ (recall that $p^+$, $i$ and $l$ are the ASP representations of $P^+$, $ I $ and $L$ respectively).
\end{definition}

\begin{definition}[Trace of an answer set] The \emph{trace of an answer set $ Z $} is the trace $\langle O_{\initialChoiceSymbol}@\initialChoiceSymbol, O_1@I_1, \dots, O_n@I_n \rangle$ where $O_{\initialChoiceSymbol}@\initialChoiceSymbol, O_1@I_1, \dots, O_n@I_n$ are exactly the manifest choice elements in $ Z $ ordered according to instants $I_1, \dots, I_n$.
\end{definition}

\begin{proposition}
A candidate trace $tr$ is a trace of $\mathcal{D}$ if and only if there exists an answer set $Z_{\tr}$ of ${\Pi_\mathcal{D} \cup \Pi_I}$ such that $tr$ is a trace of $ Z_{\tr} $.
\end{proposition}

\begin{proof} Let $ \Pi $ be the ground program obtained by grounding $ \Pi_\mathcal{D} \cup \Pi_I $. We split $ \Pi $ with respect to the set $U$ of all possible groundings of the predicates $\fluent$, $ \action $, $ \instant $ and $\possVal$. The bottom $ \botProgram{U}{\Pi} $ is guaranteed by the translation process to have a unique answer set $Z_\mathcal{L}$ which includes a correct representation of the domain language $ \mathcal{L} $, i.e. of the three sorts $ \mathcal{F} $, $ \mathcal{A} $ and $ \mathcal{I} $ and of the function $\vals$ (note that the definition of $ \mathcal{V} $ is implicitly derived from that of $ \vals $ and that our implementation is restricted to the case where $ \leq = \leq_{\mathbb{N}} $ and $ \initialInstant = 0 $).
	

We now split the partially evaluated top $ \Pi^{(1)} = \evaluation{U}{\topProgram{U}{\Pi},Z_\mathcal{L}} $ using the set $U^{(1)}$ consisting of all possible groundings of the predicate $ \holds $. The bottom $\botProgram{U^{(1)}}{\Pi^{(1)}}$ consists only of (PEC7) and has answer sets that correspond to any possible world in the domain language, i.e., (PEC7) generates every possible function from instants to states, hence for a particular world $W\in\mathcal{W}$ we denote by $Z_W$ the corresponding answer set of $\botProgram{U^{(1)}}{\Pi^{(1)}}$, and we are allowed to interpret $\holds( ((x,v),i) ) \in Z_{W}$ as $\lit{X}{V}\in W(I)$.

Notice that for any fixed $W\in \mathcal{W}$ the three sets of propositions $\Pi^{(2)} = \evaluation{U^{(1)}}{\text{(PEC8--9)} \cup C_\mathcal{D},Z_W}$, $\Pi^{(3)} = \evaluation{U^{(1)}}{\text{(PEC10)}\cup I_\mathcal{D},Z_W}$ and $\Pi^{(4)} = \evaluation{U^{(1)}}{\text{(PEC17--18)} \cup P_\mathcal{D},Z_W}$ are independent of each other so we can evaluate their answer sets separately.

We start with the set $\Pi^{(2)}$. If a c-proposition $c=$``$\theta \causesOneOf \{ O_1, O_2, \dots, O_m \}$'' in $\mathcal{D}$ is activated at $ I $ in $W$ w.r.t. $ \mathcal{D} $, i.e. $W\mmodels [\theta]@I$, then also the preconditions of the translated c-proposition in $ C_\mathcal{D} $ are satisfied (as $Z_W$ correctly represents $W$), and $\Pi^{(2)}$ will contain the facts $\causesOutcome( (id^n_j, p), i )$ for $1\leq j\leq m$ and $c$ being the $n$th c-proposition in the enumeration fixed during the translation process (see Section \ref{sec:translationdomaindependent} for reference), alongside the corresponding $\belongsTo$ facts in $ C_\mathcal{D} $ which we assume correctly represent the $\in$ relation for outcomes, i.e., $\belongsTo( (x,v), id^n_j ) \in \Pi^{(2)}$ if and only if $\lit{X}{V} \in O_j$. The converse is a straightforward inversion of this reasoning.

For a fixed $I$, if we denote by $C_{W,I}$ the set of facts of the form $\causesOutcome( (id^n_j, p), i )$ that are in $\Pi^{(2)}$, what we have just shown yields:
\begin{equation}\label{eq:CWI}
C_{W,I} \Leftrightarrow \text{a cause occurs in $W$ at $I$}
\end{equation}

If we now let $U^{(2)}$ be a splitting set such that it contains all possible groundings of the $\causesOutcome$ predicate, we get that the the only answer set of $\botProgram{U^{(2)}}{ \Pi^{(2)} }$ is $ B_C\cup C_W $, where $ B_C $ is the set of $ \belongsTo $ facts contained in $ C_\mathcal{D} $ and $ C_W $ is defined as:
\begin{equation}
	C_W = \bigcup_{I\in{\mathcal{I}}} C_{W,I}
\end{equation}

The partially evaluated top $\Pi^{(2)}_{1} = \evaluation{U^{(2)}}{\Pi^{(2)},B_C \cup C_W} $ includes the following set of facts:\\

\noindent
$O_W = \{ \inOcc(i) \mid \exists o : \causesOutcome(o,i) \in C_W \} =\{ \inOcc(i) \mid I\in\mathcal{I}, \, C_{W,I}\neq \emptyset \}$\\
\hspace*{\axiomsMedIndent}$=\{ \inOcc(i) \mid \text{a cause occurs in $W$ at $I$} \}$\\

\noindent
where we have used Equation \eqref{eq:CWI} to derive the last equality. Therefore, we can interpret $\inOcc(i)\in O_W$ as $I\in \inOcc_\mathcal{D}(W)$.

Let $U^{(2)}_1$ be the splitting set consisting of all possible groundings of $ \inOcc $. The only answer set of $\botProgram{U^{(2)}_1}{ \Pi^{(2)}_{1}}$ is $O_W$, and we now need to evaluate and find the answer sets of $\Pi^{(2)}_2 = \evaluation{U^{(2)}_1}{\topProgram{U^{(2)}_1}{\Pi^{(2)}_{1}},O_W}$ which now consists only of a partially evaluated (PEC7).

The role of (PEC7) is to implement the $\effectChoice$ function. Indeed, for each instant $I$ such that $I\in occ_\mathcal{D} (W)$, exactly one atom of the form $\effectChoice( o , i)$ is included in an answer set of $\Pi^{(2)}_2$ for some $o$ such that $\causesOutcome( o, i ) \in C_{W}$. Since this is consistent with definition \ref{def:effectChoice}, we can interpret $ \effectChoice(o,i) $ as its intended semantic counterpart $\ec(I)=O$ where $\ec$ is an effect choice function for $W$ w.r.t. $\mathcal{D}$. For an effect choice function $\ec$ for $W$ w.r.t. $\mathcal{D}$, we call the corresponding answer set that encodes it $E_{\ec}$.

Applying the splitting theorem, we can now conclude that answer sets of $\Pi^{(2)}$ are exactly those given by the set $ \{ B_C \cup C_W \cup O_W \cup E_{\ec} \mid \ec$ is an effect choice function for $W$ w.r.t. $ \mathcal{D} \}$.

Answer sets of $\Pi^{(3)}$ correspond to the $\initialChoice$ constant and can be worked out in a similar way as in the effect choice function case. It can be shown that $\initialChoice(o)$ correctly represents an initial choice $\ic$ as in definition \ref{def:initialCondition}, and answer sets of $\Pi^{(3)}$ are given by $I_\mathcal{D} \cup I_{\ic}$, where the singleton set $I_{\ic}$ consists only of an encoded $\ic$ for an initial choice $\ic$ w.r.t. $\mathcal{D}$.

Finally we need to derive answer sets of $ \Pi^{(4)} $. We split it using $ U^{(4)} $ consisting of all $ \performed $ ground facts. The bottom $\botProgram{U^{(4)}}{ \Pi^{(4)} }$ has answer set $ P_\mathcal{D} $ itself (notice that $ P_\mathcal{D} $ contains only ground facts), and we are left with calculating answer sets of $ \Pi^{(4)}_1 = \evaluation{{U^{(4)}}}{\topProgram{U^{(4)}}{\Pi^{(4)},P_\mathcal{D}},P_\mathcal{D}} $. The aim of $ \eval $ is that of implementing Equation \eqref{eq:evalpproposition}. It is important to notice here that, thanks to requirement iv) in Definition \ref{def:domaindescription}, it is possible to label a p-proposition ``$ A \performedAt I \withProb P^+ $'' using only $ A $ and $ I $. Comparing Equation \eqref{eq:evalpproposition} with (PEC17--18) immediately gives that the only answer set of $ \Pi^{(4)} $ is $ P_\mathcal{D} \cup {\Ev}_W $ where \\
$\Ev_W = \Big\{ \eval(a,i,p) \mid A\in\mathcal{A}, I\in\mathcal{I}, \text{``$ A \performedAt I \withProb P^+ $''} \in \mathcal{D},$\\
\hspace*{\axiomsBigIndent}$P=P^+ \text{ if } W\mmodels \ifor{A}{I}, P=1-P^+ \text{ otherwise} \Big\}$.

We are now able to calculate the answer sets of the whole program $\Pi \setminus \Pi_C$, which are given by the set $\Big\{ Z_\mathcal{L} \cup Z_W \cup B_C \cup C_W \cup O_W \cup E_{\ec} \cup I_{\mathcal{D}} \cup I_{\ic} \cup P_\mathcal{D} \cup \Ev_W$, for $ W\in \mathcal{W} $, an effect choice $\ec$ for $W$ w.r.t. $\mathcal{D}$ and an initial choice $\ic$ w.r.t. $\mathcal{D} \Big\}$

Finally, we take into account the constraints $ \Pi_C $, whose effect is that of implementing the Closed World Assumption and the effects of initialisation and persistence. Since (PEC11--16) are constraints, they eliminate those answer sets of $\Pi \setminus \Pi_C$ that satisfy their bodies.

If we let $Z_W$ be the world encoded in an answer set $Z$ of $\Pi \setminus \Pi_C$, (PEC11) and (PEC12) ensure that:\\

\noindent
$\holds(((a,\true),i)) \in Z_W \Rightarrow \exists p>0, \, \performed(a,i,p) \in \Pi_\mathcal{D}$\\
\hspace*{\axiomsMedIndent}$\Leftrightarrow \text{``$A \performedAt I \withProb P^+$''} \in \mathcal{D}$  \\

\noindent
and, conversely

\noindent
$\holds(((a,\false),i)) \in Z_W \Rightarrow \performed(a,i,1) \in \Pi_\mathcal{D}$\\
\hspace*{\axiomsMedIndent}$\Leftrightarrow \text{``$A \performedAt I \withProb 1$''} \in \mathcal{D}$  \\

\noindent
therefore the world encoded in $Z_W$ must satisfy CWA, i.e. definition \ref{def:cwa}.

Let now ``$\initiallyOneOf ( O_1, O_2, \dots, O_m )$'' be an i-proposition in $ \mathcal{D} $ and $Z_W$ be as before. (PEC14) makes sure that:\\

\noindent
$\holds(((f,v),0)) \in Z_W \Leftrightarrow \exists s : \{ \initialChoice( (s,p) ), \belongsTo( (f,v), s ) \} \subseteq \Pi_\mathcal{D} \Leftrightarrow $\\
$\exists O : O\in \{O_1, \dots, O_m\}, \tilde{S}=\chi(O), [F=V]\in \tilde{S}. $\\

\noindent
which satisfies the initial condition, i.e. definition \ref{def:initialChoice}.

Finally we consider (PEC14--16). Let $I,I'$ be two instants with $I<I'$ as in definition \ref{def:justifiedChange}, consider the world encoded in $Z_W$ and let $W(I)\restriction\mathcal{F}=\tilde{S}$ and $W(I')\restriction\mathcal{F}=\tilde{S}'$. Assume that the $ \effectChoice $ function encoded in $Z$ maps instants in $\inOcc_\mathcal{D}(W) \cap [I,I')$ to outcomes $O_1, O_2, \dots, O_n$. Axiom (PEC16) makes sure that $\tilde{S}$ cannot be altered if $I \notin \inOcc_\mathcal{D}(W)$. Therefore $\tilde{S}$ can only change at instants $I \in \inOcc_\mathcal{D} (W)$. We now show that $\tilde{S}'$ is actually equal to $\tilde{S} \oplus O_1 \oplus O_2 \oplus \dots \oplus O_n$. If not, and considering that our implementation is restricted to a finite set of instants, either (i) there is a fluent literal $L \in \chi(O)$ for some $O \in \{ O_1, O_2, \dots, O_n \}$ and an instant $I'' \in [I,I')$ such that $ L \in \chi(O) $ but $L \notin W(I'' + 1)$, or (ii) for some $O \in \{ O_1, O_2, \dots, O_n \}$ and a fluent literal $L=[F=V]$ such that $L \notin O$ and $L \notin W(I'')$, $L \in W(I''+1)$. Both (i) and (ii) are forbidden by (PEC15) and (PEC16) respectively, by considering that the answer set $Z$ correctly represents the semantic objects that it encodes.
\end{proof}

\section{Related Work}
Although there is existing work on probabilistic reasoning about actions, most is based on Reiter's variant of  Situation Calculus (SC) \cite{reiter2001knowledge}, with focus on hypothetical rather than narrative reasoning. An exception is Prob-EC (see below).

Of the SC approaches, the Bacchus-Halpern-Levesque framework \cite{bacchus1999reasoning} is a cornerstone of early work integrating probabilistic knowledge with logical formalisms for reasoning about actions, and incorporates epistemic notions such as sensing actions. The Probabilistic Situation Calculus (PSC) \cite{mateus2001probabilistic} is extended to deal with knowledge-producing actions in \cite{mateus2002observations}. A reasoning system based on PSC able to perform temporal projection has been implemented by the authors in Wolfram Mathematica \cite{wolfram2003mathematicabook} and uses Monte Carlo methods for tractability. The language PAL \cite{baral2002reasoning} focuses on building an elaboration tolerant representation for Markov Decision Processes. It is based on Language $ \mathcal{A} $ \cite{gelfond1993representing} and oriented to counterfactual reasoning and observation assimilation. PAL uses two  kinds of unknown variables -- inertial and non-inertial -- to achieve an elaboration tolerant representation of domains. The action language $ \mathcal{E}+ $ \cite{iocchi2009e}, based on $ \mathcal{C}+ $ \cite{giunchiglia2004nonmonotonic},  supports both non-deterministic and probabilistic actions. Its main focus is on providing algorithms for the efficient computation of plans.
	
	To our knowledge, the Probabilistic Logic Programming Event Calculus (Prob-EC)~\cite{skarlatidis2015probec} is the only EC-style language in this class of formalisms other than PEC able to support reasoning about explicit event occurrences (narratives). Unlike our framework, which has its own bespoke semantics, Prob-EC is a logic programming framework based on the probabilistic logic programming language ProbLog \cite{deraedt2007problog} and therefore inherits and exploits its semantics. In~\cite{skarlatidis2015probec} Prob-EC is applied to human activity recognition. The authors describe how a set of long-term activities (LTAs) can be detected from a set of short-term activities (STAs). Such STAs, which constitute the input to the system, are treated as events happening at given instants and have probabilities attached.  This is a somewhat different approach than PEC's, motivated by its application to activity recognition, analogous to attaching probabilities to p-propositions (rather than i- and c-propositions). In other words Prob-EC's focus is on representing probabilistic knowledge about event occurrences rather than about their general causal effects.   


\section{Summary}
In this work, we present PEC, an EC variant for reasoning about actions in a narrative domain where actions can have probabilistic outcomes, and illustrated how for a wide sub-class of domains it can be implemented in ASP in a sound and complete way. Unlike Prob-EC \cite{skarlatidis2015probec} which follows the ``logic programming'' tradition, our formalism belongs to the ``action language'' tradition (originating in \cite{gelfond1993representing}, but see also \cite{kakas1997languageE} for the first EC style action language), and therefore its own specialised semantics. This makes of PEC portable in the sense that it is independent of any particular computational implementation. Its semantics is defined in terms of (possible) \emph{worlds}, with a view to adding epistemic features at a later date (see e.g.~\cite{moore1985knowledge}, \cite{scherl2003knowledge}).

In our initial experimentation with adding  epistemic features to PEC, we have focused on representing \emph{imperfect sensing actions} and \emph{actions conditioned on knowledge} acquired during  the progression of the narrative. These features are similar to those in the EFEC extension of FEC~\cite{miller2014epistemic}.  We envisage including \emph{s-propositions } such as 

\[
	\See \senses \Coin \withAccuracies \begin{pmatrix} 0.9\; & 0.1\\ 0.3\; & 0.7 \end{pmatrix}
 \]

\noindent which represents that our coin-tossing robot can imperfectly sense the current face showing on the coin, and \emph{conditional p-propositions} such as

\[
	\Toss \performedAt 2 \ifBelieves ( \lit{\Coin}{\Tails}, (0.65,1] )
 \]

\noindent which represents that the robot will toss again if it believes with a greater than 65\% probability that the first toss resulted in $\Tails$.  Preliminary results indicate that our possible worlds semantics can be readily extended to cover these notions.

There are several other ways in which the present work can be continued. For instance, the problem of \emph{elaboration tolerance}, which plays an important role in classical reasoning about actions, needs to be reviewed and solved in our setting. This problem has already been tackled in \cite{baral2002reasoning}, but needs to be restated in our framework due to the different way in which we introduce probabilities in PEC. A related point is that of \emph{underspecification}, i.e. what an agent can reasonably infer from a domain in which the initial conditions and the effects of actions are not entirely specified (even probabilistically). Finally, in our view a crucial point is that of \emph{computational efficiency}. Indeed, the intractability of several computational problems arising in this setting (such as temporal projection) suggests that techniques (e.g. Monte Carlo Markov Chain) are needed to efficiently approximate the correct answer to a given query with an appropriate degree of confidence.

\bibliography{phd}{}
\bibliographystyle{plain}

\end{document}